 \documentclass[ijoc,sglanonrev]{informs4}
\PassOptionsToPackage{table}{xcolor}
\usepackage{eqndefns-left} % For checking the display equation width and equation environment definitions %
\RequirePackage{tgtermes}
\RequirePackage{newtxtext}
\RequirePackage{newtxmath}
\RequirePackage{bm}
\RequirePackage{endnotes}

\OneAndAHalfSpacedXII % Current default line spacing
\usepackage{algorithm}
\usepackage{algpseudocode}
\usepackage{tikz}
\usepackage[T1]{fontenc}

\usepackage{enumitem}
\usepackage{amsmath}
\newtheorem{definition}{Definition}
\newtheorem{proposition}{Proposition}
\newtheorem{remark}{Remark}
\def\ECHowTheorems{%
  \setcounter{definition}{0}\def\thedefinition{EC.\arabic{definition}}%
  \setcounter{proposition}{0}\def\theproposition{EC.\arabic{proposition}}%
  \setcounter{remark}{0}\def\theremark{EC.\arabic{remark}}%
}
\usepackage{graphicx}
\usepackage{booktabs}
\usepackage{multirow}
\usepackage{pifont}% http://ctan.org/pkg/pifont
\usepackage{caption,subcaption}
\usepackage{adjustbox}
\usepackage{soul, color, xcolor}
\definecolor{myorange}{RGB}{255,165,0}

\usepackage[table]{xcolor} % 导入 colortbl 宏包
\usepackage{caption,subcaption}
\usepackage{tcolorbox}
\tcbuselibrary{skins, breakable}
\usepackage{xcolor}
\definecolor{mygray}{RGB}{192,192,192}
\usepackage{enumitem}
\usepackage{amsmath}

\usepackage{booktabs}
\usepackage{multirow}
\usepackage{pifont}% 
\usepackage{adjustbox}
\usepackage{booktabs}
\usepackage{soul, color, xcolor}
\definecolor{myorange}{RGB}{255,165,0}
\usepackage{makecell}
\usepackage[table]{xcolor} % 导入 colortbl 宏包

\usepackage{algorithm}
\usepackage{algpseudocode}

\usepackage{xcolor}
\definecolor{mygray}{RGB}{192,192,192}

\usepackage{float}
\usepackage{xspace}
\tcbset{
  aibox/.style={
    width=\textwidth,
    top=10pt,
    colback=white,
    colframe=black,
    colbacktitle=black,
    enhanced,
    center,
    attach boxed title to top center={yshift=-0.1in},
    boxed title style={boxrule=0pt,colframe=white,},
  }
}
\newtcolorbox{AIbox}[2][]{aibox,title=#2,#1}
\usepackage{listings}
\definecolor{myblue}{RGB}{100, 150, 200}
\definecolor{mygreen}{RGB}{80, 160, 80}

\definecolor{darkgreen}{rgb}{0.0, 0.5, 0.0}
\definecolor{darkgray}{gray}{0.4}
\definecolor{maroon}{rgb}{0.5, 0.0, 0.0}
\definecolor{navy}{rgb}{0.0, 0.0, 0.5}
\definecolor{teal}{rgb}{0.0, 0.5, 0.5}

\usepackage{float}
\usepackage{xspace}
\tcbset{
  aibox/.style={
    width=\textwidth,
    top=10pt,
    colback=white,
    colframe=black,
    colbacktitle=black,
    enhanced,
    center,
    attach boxed title to top center={yshift=-0.1in},
    boxed title style={boxrule=0pt,colframe=white,},
  }
}

\usepackage{listings}
\definecolor{myblue}{RGB}{100, 150, 200}
\definecolor{mygreen}{RGB}{80, 160, 80}

\definecolor{darkgreen}{rgb}{0.0, 0.5, 0.0}
\definecolor{darkgray}{gray}{0.4}
\definecolor{maroon}{rgb}{0.5, 0.0, 0.0}
\definecolor{navy}{rgb}{0.0, 0.0, 0.5}
\definecolor{teal}{rgb}{0.0, 0.5, 0.5}

\definecolor{promptgray}{gray}{0.85}
\definecolor{promptbody}{gray}{1.0}

\newtcolorbox{promptbox}[1]{
  enhanced,
  colback=promptbody,
  colframe=black,
  boxrule=0.6pt,
  arc=2pt,
  left=8pt, right=8pt, top=6pt, bottom=6pt,
  fonttitle=\bfseries,
  coltitle=black,
  colbacktitle=promptgray,
  attach boxed title to top left={xshift=0pt, yshift=0pt},
  boxed title style={colframe=black, boxrule=0.6pt, arc=2pt},
  title={#1},
  breakable
}
\usepackage{amsmath,amsfonts,bm}

\def\eqref#1{equation~\ref{#1}}
\def\1{\bm{1}}

\DeclareMathAlphabet{\mathsfit}{\encodingdefault}{\sfdefault}{m}{sl}
\SetMathAlphabet{\mathsfit}{bold}{\encodingdefault}{\sfdefault}{bx}{n}

\usepackage{hyperref}
\usepackage{url}
\usepackage{algorithm}

\usepackage{natbib}
 \bibpunct[, ]{(}{)}{,}{a}{}{,}%
 \def\bibfont{\small}%
\EquationsNumberedThrough    % Default: (1), (2), ...
\ECRepeatTheorems  %  

\MANUSCRIPTNO{}

\def\theARTICLETOP{}
\RRHSecondLine{}
\LRHSecondLine{}
\makeatletter
\def\theARTICLEABSTRACT{%
  \HOOKb
  \vspace*{18pt}
  \begin{minipage}[t]{\textwidth}\parindent1em
    \ABSfont
    \noindent\theABSTRACT\endgraf
    \vskip5pt
    \theFUNDING
    \theKEYWORDS
    \theSUBJECTCLASS
    \theAREAOFREVIEW
    \theMSCCLASS
    \theORMSCLASS
    \if@BLINDREV\else\theHISTORY\fi
    \noindent\hrulefill
  \end{minipage}%
  \vspace*{0pt}
}
\makeatother
\begin{document}
%%%%%%%%%%%%%%%%

%\setcounter{page}{1}
%\DOI{10.1287/xxxx.0000.0000}%

% Double-anonymous submission: author identities are withheld in the running head.
\RUNAUTHOR{Kang et al.}

% \RUNTITLE{Predictive Maintenance in Manufacturing}
\RUNTITLE{Observability-Safe Memory Retention for Long-Horizon Language Agents}

% \TITLE{Optimal Resource Allocation in Humanitarian Logistics: A Stochastic Programming Approach}
\TITLE{Learning What to Remember: Observability-Safe Memory Retention via Constrained Optimization for Long-Horizon Language Agents}

\ARTICLEAUTHORS{%
%\AUTHOR{John Doe,\textsuperscript{a} Jane Smith,\textsuperscript{b}}
%\AFF{\textsuperscript{a}Department of Industrial Engineering, University of XYZ, \EMAIL{john.doe@xyz.edu; \textsuperscript{b}Department of Computer Science, University of ABC, \EMAIL{jane.smith@abc.edu}} 
\AUTHOR{Qingcan Kang\textsuperscript{1,*}, Mingyang Liu\textsuperscript{2,*}, Shixiong Kai\textsuperscript{3}, Kaichao Liang\textsuperscript{3}, Tao Zhong\textsuperscript{3}, Mingxuan Yuan\textsuperscript{3,\dag}}
\AFF{\textsuperscript{1}Department of Industrial Engineering and Decision Analytics, Hong Kong University of Science and Technology;
\textsuperscript{2}Department of Computer Science, City University of Hong Kong;
\textsuperscript{3}Huawei Noah's Ark Lab\\[2pt]
\textsuperscript{*}Equal contribution.\quad
\textsuperscript{\dag}Corresponding author.}
} % end of the block

\ABSTRACT{%
Long-horizon language agents accumulate observations, reasoning traces, and retrieved facts that exceed context windows, making memory retention a fundamental resource-allocation problem. Existing memory systems treat retention as a local problem and do not model long-term consequences under observability constraints. To fill this gap, we formulate memory retention as a constrained stochastic optimization problem with budget feasibility, evidence utility, and delayed costs including miss, reacquisition, and stale penalties. We show that this multi-step optimization problem is NP-hard, making exact solution computationally intractable. Moreover, deployment decisions must be made under partial observability. To address these challenges, we propose \textbf{OSL-MR} (Observability-Safe Learning for Memory Retention), a learning-based framework that enforces a strict separation between online-observable features and offline-available supervision. OSL-MR combines an evidence learner trained from realized evidence with a Mixed-Score heuristic that serves as a deployable online-safe baseline and an inductive prior. The policy learns to identify query-conditioned evidence from interaction data and remains deployable using only online-observable features. Experiments on LoCoMo and LongMemEval show that OSL-MR outperforms recency-based, Generative Agents-style, and other heuristic baselines, especially under tight budgets. The Mixed-Score prior improves precision and recall, and sensitivity analysis shows robustness across cost settings. On small, exactly solvable instances, we demonstrate that even a perfect single-step optimizer cannot anticipate future demand shifts, while OSL-MR stays significantly closer to the dynamic-programming optimum, confirming the necessity of the sequential formulation and reinforcing the effectiveness of our approximation. These results establish constrained stochastic optimization and optimization-guided learning as a principled foundation for memory management in long-horizon agents.
}%

%\FUNDING{This research was supported by [grant number, funding agency].}

%Supplemental Material:
%Data Ethics & Reproducibility Note:
% Reproducibility: code, data splits, configurations, and scripts that regenerate
% all tables and figures (including the exact-DP optimality study of
% Section~\ref{sec:optimality}) are provided in an anonymized repository for review
% and will be deposited in the IJOC software repository upon acceptance; see
% Section~\ref{sec:reproducibility}.

%\KEYWORDS{Stochastic programming, Decision support,Uncertainty, Disaster response, Optimization}

%\HISTORY{Received: Month DD, YYYY; Accepted: Month DD, YYYY; Published Online: Month DD, YYYY}

\maketitle
%%%%%%%%%%%%%%%%%%%%%%%%%%%%%%%%%%%%%%%%%%%%%%%%%%%%%%%%%%%%%%%%%%%%%%

% Text of your paper here

\section{Introduction}
\label{sec:introduction}

Long-horizon language agents, built on large language models (LLMs), require memory to maintain continuity across extended interactions, reuse established facts, and adapt to user preferences \citep{zhang2024surveymemorymechanismlarge, du2026memoryautonomousllmagentsmechanisms, jiang2026anatomyagenticmemorytaxonomy}. In principle, an agent could store every past utterance, but this is rarely practical: long conversations exceed context-window or storage budgets, retrieval over large stores increases cost, and outdated memories may lead to incorrect responses \citep{liu2025comprehensivesurveylongcontext}. A memory system therefore faces a persistent retention problem: under a limited budget, it must decide which memories to keep for future interactions and which to discard. This problem is inherently long-horizon: a memory irrelevant now may become essential later, while a recent or salient memory may never support a future answer.

Most existing memory systems rely on heuristic scoring, retrieval optimization, or learned compression \citep{park2023generative, packer2023memgpt, zhong2024memorybank, jiang2023llmlingua}, but treat retention as a local decision that optimizes immediate relevance or similarity. They are myopic: they do not model how the retained cache carries over across interactions, how evicting a memory can create delayed misses or costly reacquisition, or how retained information can become stale. They also lack a principled formulation of the long-horizon consequences---missing future evidence, reacquisition cost, and stale information under partial observability---and do not specify what an optimal retention policy should achieve under limited capacity and delayed feedback.

\textbf{We fill this gap by formulating memory retention as a constrained stochastic optimization problem} that explicitly accounts for budget feasibility, evidence utility, and delayed costs including miss penalty, reacquisition delay, and stale-information risk. Rooted in classical operations research (OR), this formulation views retention as sequential resource allocation under uncertainty, with a hard budget and outcome-dependent costs that materialize only in the future. To our knowledge, existing work has not provided a unified formulation that
simultaneously captures hard budget constraints, delayed retention
consequences, reacquisition cost, stale-information risk, and partial
observability. Our formulation makes these assumptions explicit and yields a
natural learning objective for deployable retention policies. The closest optimization-based approach, Fofadiya \& Tiwari \citep{fofadiya2026forgetting}, treats retention as a \emph{single-step} problem that optimizes immediate relevance, ignoring how current decisions affect future evidence availability and incur delayed penalties, and thus misses the essential long-horizon nature of the problem.

A central challenge in operationalizing this formulation is observability. Many signals needed to evaluate a retention decision---gold evidence, answer correctness, semantic freshness---are available only after the decision is made. Using them at deployment would create an unrealistic information advantage. We therefore enforce a strict separation between \emph{online-observable features} (query context, memory metadata, interaction history) and \emph{offline-available supervision (OAS)} (gold evidence, answer text, ground-truth freshness). OAS is used only for training and evaluation, while deployable policies rely solely on online inputs. This discipline guarantees that any policy learned under it can be deployed without oracle access to future information, making it suitable for online interactive agents.

Building on this formulation and observability discipline, we propose
\textbf{OSL-MR} (\textbf{O}bservability-\textbf{S}afe \textbf{L}earning for \textbf{M}emory \textbf{R}etention), a staged framework for constrained memory retention in long-horizon language agents. Since the full multi-step retention
problem is NP-hard (Proposition~\ref{prop:nphard}) and deployment decisions must be made under partial observability, exact optimization is generally
impractical in realistic long-horizon settings. OSL-MR provides a \textbf{learning-based approximation framework} with two complementary components: (i) an evidence learner trained offline from interaction logs using supervision derived from realized evidence, and (ii) a Mixed-Score heuristic that serves as both a cold-start deployable baseline and an online-safe feasibility prior. Deployment proceeds in two stages. Mixed-Score operates as the standalone policy from the first user query while all interactions are logged; once sufficient data accumulate, the evidence learner is trained offline and deployed as a frozen policy, replacing the heuristic without violating observability. By learning query-conditioned evidence scores directly from realized evidence labels, OSL-MR avoids relying on hand-designed importance labels or general-purpose importance estimators. The Mixed-Score value can also be fed to the learner as an additional scalar feature, serving as an online-safe inductive prior.

A key design choice is to learn from evidence-membership signals rather than from direct reward. The reward is inherently non-decomposable and combinatorial, since retention quality depends on sets of memories, not individual items, so assigning reward to individual decisions is ill-posed and computationally intractable. Evidence membership instead provides a tractable per-memory signal that aligns with the structure of the optimization objective while remaining fully compatible with online deployment.

Our contributions are as follows:
\begin{itemize}
    \item \textbf{Problem formulation.} We formalize memory retention as a constrained sequential decision problem under a hard budget, modeling evidence utility, storage cost, miss penalty, reacquisition delay, and stale risk, and prove that the full multi-step problem is NP-hard (Proposition~\ref{prop:nphard}), establishing the need for approximation.
    \item \textbf{Observability-safe learning framework.} We introduce OSL-MR, which enforces a strict online/OAS separation and integrates the Mixed-Score prior, the optimization formulation, and an evidence learner trained from direct evidence supervision.
    \item \textbf{Empirical gains.} On LoCoMo and LongMemEval, OSL-MR consistently outperforms recency-based methods, Generative Agents-style scoring, Mixed-Score, and behavior-cloning variants, especially under tight budgets. Ablation results show that the Mixed-Score prior improves precision and modestly improves recall, and sensitivity analysis confirms robustness across cost configurations.
    \item \textbf{Optimality analysis on solvable instances.} On small, exactly-solvable instances of the same problem, we demonstrate that single-step optimization is insufficient to anticipate future demand shifts: a perfect single-step optimizer is far from optimal under shifting demand, while OSL-MR stays significantly closer to the dynamic-programming optimum, confirming the necessity of the sequential formulation and reinforcing the effectiveness of our learning-based approximation.
\end{itemize}
\section{Related Work}
\label{sec:related}

\subsection{Memory Systems and Long-Horizon Language Agents}

Long-horizon language agents rely on external memory to extend their effective context. Early systems store episodic experiences or tool traces in vector databases; MemGPT \citep{packer2023memgpt} introduces operating-system-inspired hierarchical paging, while MemoryBank \citep{zhong2024memorybank} adds Ebbinghaus-style forgetting to retrieval. Generative Agents \citep{park2023generative} rank memories by recency, relevance, and importance, but their static importance scores capture general salience rather than query-specific evidence value, limiting effectiveness under capacity constraints. Other systems extend different stages of the memory lifecycle: Mem0 \citep{chhikara2025mem0} structures memory writing, MEM1 \citep{zhou2025mem1} learns compact latent representations via RL, and prompt compression \citep{jiang2023llmlingua} reduces context cost. Across these systems, retention is typically coupled with retrieval,
summarization, or compression rather than treated as a standalone sequential decision problem. To our knowledge, existing work has not provided a unified formulation that jointly captures hard retention budgets, delayed evidence
utility, stale-memory risk, and partial observability---the gap we address.

\subsection{Memory Retention as Resource Allocation}

Memory retention is a resource-allocation problem under uncertainty, classically modeled in operations research as constrained optimization that trades off immediate gains against future penalties; surveys \citep{hu2025memory, huang2026rethinkingmemorymechanismsfoundation} note the absence of such unified formulations for retention. Several works add optimization for cost-constrained retrieval (e.g., AdaGReS \citep{adagres2025}, CORAG \citep{wang2024corag}), but target one-turn context selection rather than long-horizon retention. For retention itself, Fofadiya and Tiwari \citep{fofadiya2026forgetting} formulate a single-step budgeted optimization that maximizes immediate relevance independently at each step, ignoring how current decisions affect future evidence availability or incur delayed penalties.

At its combinatorial core, the per-step retention subproblem generalizes budgeted maximum coverage, a classical NP-hard problem \citep{khuller1999budgeted, feige1998threshold}; we make this precise in Proposition~\ref{prop:nphard}. The corresponding known-coverage objective is monotone submodular, so greedy algorithms attain $(1-1/e)$-type guarantees under knapsack constraints when the coverage sets are known \citep{nemhauser1978analysis, khuller1999budgeted, feige1998threshold} (detailed in the online supplement, Remark~\ref{rem:greedy}). Our stance follows directly: at deployment the coverage sets are unobservable, so the guarantee is not directly attainable, which is exactly why we pursue a greedy, learning-augmented approximation that estimates evidence value from logged supervision rather than solving the problem exactly.

Classical OR methods such as integer or dynamic programming do not apply directly here: the problem is partially observable---the distribution of future evidence demands depends on unknown user queries---and the state space is combinatorial. On deliberately reduced instances, however, it becomes exactly solvable by dynamic programming, which we exploit in Section~\ref{sec:optimality} to obtain ground-truth optima and measure how near-optimal a deployable policy can be. At benchmark scale exact optimization is out of reach, so we adopt a learning-based approximation: OSL-MR is derived from a multi-step sequential retention objective, but learns deployable evidence scores from logged supervision rather than solving the exact dynamic program or directly optimizing a non-decomposable reward online. By contrast, BudgetMem approaches
\citep{alla2026budgetmemlearningselectivememory,
zhang2026learningqueryawarebudgettierrouting} primarily focus on one-step budgeting or routing decisions, while Mem-T
\citep{yue2026memtdensifyingrewardslonghorizon} and MemAct
\citep{zhang2026memoryactionautonomouscontext} consider long-horizon agent behavior but do not formulate retention as a multi-step budgeted optimization problem with partial observability and online/OAS separation. Table~\ref{tab:comparison} compares representative approaches across five dimensions---constrained optimization, delayed feedback, partial observability, online/OAS separation, and long-horizon sequential view; OSL-MR is the first to satisfy all five.

\subsection{Learning-Based Memory Policies and Observability Separation}

Recent learning-based systems optimize memory policies using downstream signals: Mem-$\alpha$ \citep{wang2025memalpha} applies RL to learn memory construction, Acon \citep{kang2025acon} optimizes context compression for long-horizon agents, and MemRL \citep{zhang2026memrl} frames retrieval as a value-based decision updated from environmental feedback. These improve retrieval or compression rather than retention under hard budgets. The STALE benchmark \citep{chao2026stale} shows that LLM agents struggle to detect when a stored memory has become outdated, motivating a separation between observable temporal signals and latent semantic validity---precisely our online/OAS principle. Whereas many learning methods rely on supervision unavailable at deployment, OSL-MR uses gold evidence only during offline training and restricts deployed policies to online-observable features, preserving validity under partial observability. Our framework is complementary to retrieval-side optimization (e.g., MemRL); unifying retention and retrieval under a single constrained optimization formulation is a promising direction for future work.

\begin{table}[t]
\centering
\small
\caption{Comparison of representative memory-related approaches. Criteria refer specifically to retention under limited capacity rather than retrieval, compression, or routing in general. OSL-MR is the only compared method that simultaneously satisfies all five criteria.}
\label{tab:comparison}
\begin{tabular}{lccccc}
\toprule
\multirow{2}{*}{\textbf{Method}} & \textbf{Budgeted} & \textbf{Delayed} & \textbf{Partial} & \textbf{Online/OAS} & \textbf{Long-Horizon} \\
& \textbf{Retention} & \textbf{Feedback} & \textbf{Observability} & \textbf{Separation} & \textbf{Sequential} \\
\midrule
Generative Agents & $\times$ & $\times$ & $\times$ & $\times$ & $\times$ \\
BudgetMem & $\times$ & $\times$ & $\times$ & $\times$ & $\times$ \\
Mem-T & $\times$ & $\checkmark$ & $\times$ & $\times$ & $\checkmark$ \\
MemAct & $\times$ & $\times$ & $\times$ & $\times$ & $\checkmark$ \\
Fofadiya \& Tiwari & $\checkmark$ (single-step) & $\times$ & $\times$ & $\times$ & $\times$ \\
\midrule
\textbf{OSL-MR (ours)} & $\checkmark$ (multi-step)& $\checkmark$ & $\checkmark$ & $\checkmark$ & $\checkmark$ \\
\bottomrule
\end{tabular}
\end{table}
\section{Methodology}
\label{sec:method}

We propose OSL-MR, a memory retention framework for long-horizon language agents under strict budget constraints and partial observability. Memory decisions are made sequentially with limited capacity, yet their consequences---information loss, recomputation cost, and stale usage---surface only in the future, which makes naive heuristic rules or local scoring strategies insufficient for this delayed, partially observable problem.

OSL-MR addresses this by formulating memory retention as a constrained partially observable stochastic optimization problem under an explicit observability separation (Section~\ref{sec:observability}). The framework integrates three tightly coupled components, elaborated in the following subsections: (i) a constrained optimization formulation that defines long-horizon retention objectives under budget limitations (Section~\ref{sec:formulation}), (ii) a Mixed-Score heuristic that provides a fully deployable cold-start solution and a strong inductive baseline (Section~\ref{sec:mixed-score}), and (iii) an evidence learner trained offline from interaction logs that refines the heuristic policy using supervision derived from realized evidence (Section~\ref{sec:evidence-learning}). Section~\ref{sec:bc} discusses a behavior-cloning variant for comparison. Figure~\ref{fig:method_overview} illustrates the overall framework and its data flow.

\begin{figure}[!t]
\centering
\includegraphics[width=\columnwidth]{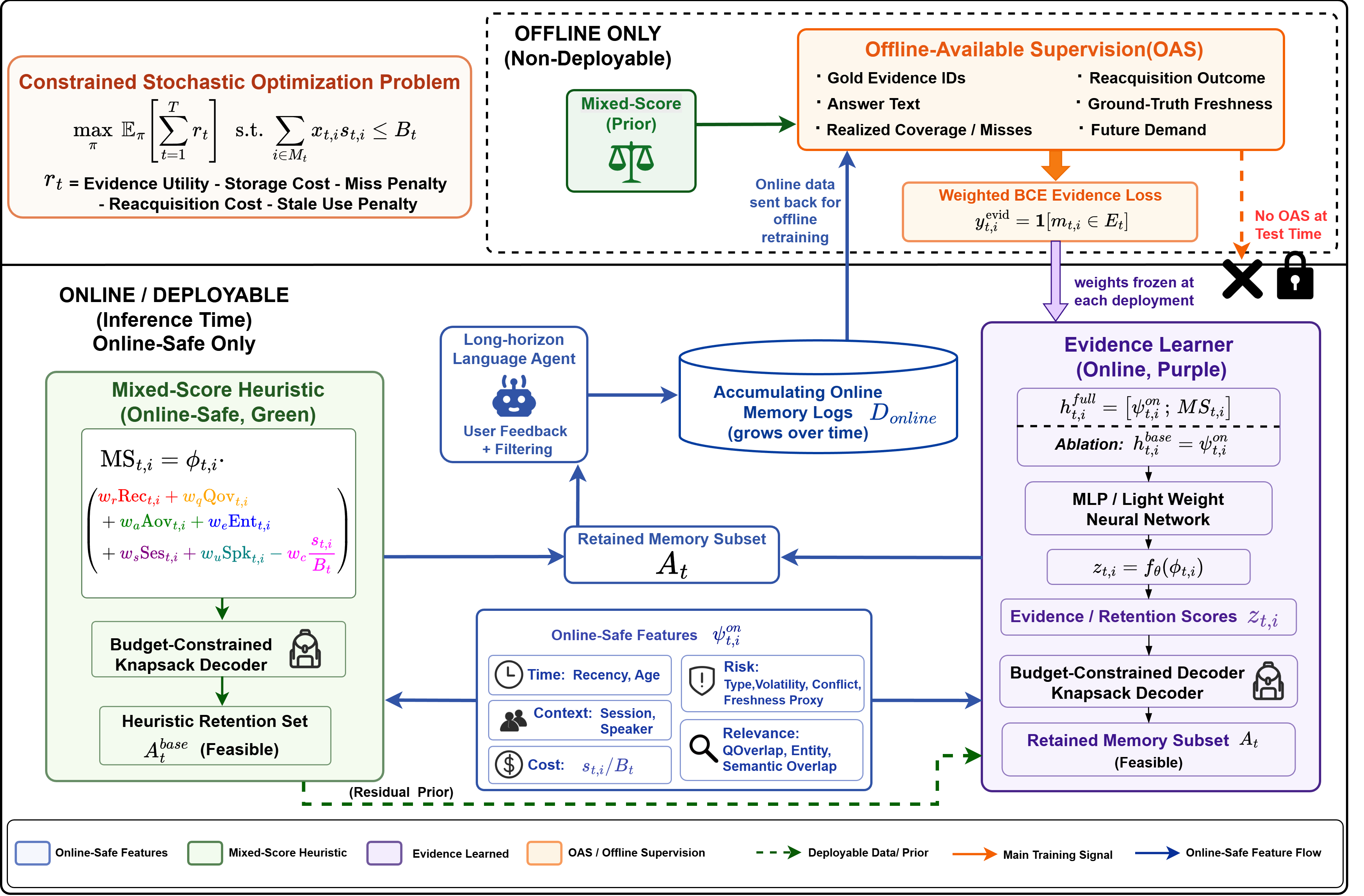}
\caption{Overview of the OSL-MR framework. The framework separates online-observable inputs (query, memory metadata, interaction history) from offline-available supervision (gold evidence, answer text). During cold-start, the Mixed-Score heuristic selects retained memories under budget $B_t$ and logs agent-user interaction data (queries, responses, and subsequent evidence outcomes). After sufficient data collection, an evidence learner is trained offline using gold evidence labels; the learned policy is then frozen and deployed, relying only on online features for inference.}
\label{fig:method_overview}
\end{figure}

\subsection{Observability Separation}
\label{sec:observability}

We first formalize the information asymmetry in memory retention, distinguishing three types of signals. \textbf{Online-observable inputs} include the current query, memory metadata, and interaction context, available at decision time. \textbf{Offline-available supervision (OAS)} includes gold evidence sets, answer content, coverage signals, and downstream outcome statistics, accessible only after the interaction completes. \textbf{Evaluation signals} are derived from realized system outcomes and used strictly for analysis and benchmarking.

A key constraint in OSL-MR is that deployed policies are strictly restricted to online-observable inputs. This ensures that inference-time decisions are realistic and do not rely on oracle information, so any improvement must arise from better generalization over observable features rather than leakage of supervision.

A subtle issue arises in modeling memory freshness: recency is directly observable, but semantic validity is not---a memory may remain valid long after creation, or become obsolete shortly after being written. We therefore estimate stale-use risk using observable proxies; the exact definition and implementation are given in Section~\ref{sec:mixed-score}.

\subsection{Constrained Memory Retention Problem}
\label{sec:formulation}

We model memory retention as a constrained, partially observable sequential decision process over a horizon of $T$ interaction steps.

\noindent\textbf{State.}
All information available at the start of decision step $t$, before any new memory is selected, constitutes the pre-decision state
\[
s_t = (C_{t-1},\,\Phi_t,\,H_t),
\]
where $C_{t-1}$ is the retained store (cache) produced by the previous step (with $C_0=\emptyset$), $\Phi_t=\{\phi_{t,i}\}$ is the vector of per-memory freshness levels, and $H_t$ is an interaction-history summary (e.g., session/speaker hit statistics). A set of newly created memories $N_t$ (utterances produced since the last decision) then arrives, so the agent chooses among the candidate pool
\[
M_t = C_{t-1}\cup N_t,
\]
i.e., the memories currently in the store plus the newcomers. Each candidate $i\in M_t$ has a size $s_{t,i}$, online-observable features $\psi_{t,i}^{on}$, and a freshness level $\phi_{t,i}\in(0,1]$ that decays with its age (Section~\ref{sec:mixed-score}).

\noindent\textbf{Action and budget.}
The agent selects a retained subset $A_t \subseteq M_t$, encoded by indicators $x_{t,i}\in\{0,1\}$, subject to a strict storage budget $\sum_{i\in M_t} x_{t,i}s_{t,i} \le B_t$.

\noindent\textbf{Transition.}
The chosen set becomes the next store, $C_t = A_t$, while evicted memories $M_t\setminus A_t$ leave the cache; a dropped memory is no longer directly selectable and, should a later query need it, can be recovered only through a delayed, costly reacquisition channel (charged via the $T_t^{\mathrm{reacq}}$ term below). The remaining state advances accordingly: freshness levels decay deterministically with elapsed time (and reset when a memory is accessed), yielding $\Phi_{t+1}$, and the interaction summary updates to $H_{t+1}$ from realized outcomes, so the next pre-decision state is $s_{t+1}=(C_t,\Phi_{t+1},H_{t+1})$. This transition is the source of long-horizon coupling: a memory discarded at step $t$ is unavailable to a later query at step $t'>t$, so an eviction made now can cause a future miss and trigger reacquisition or stale-information costs.

At step \(t\) the current query \(Q_t\) is observable and conditions the retention decision (e.g., through the question-overlap features of Section~\ref{sec:mixed-score}), whereas the gold evidence demand set \(E_t\) it induces is offline-available supervision (OAS), realized only after the interaction and used solely to define the reward and training labels, never as a policy input.
Each candidate memory contributes partial evidence coverage through a mapping \(\mathrm{cov}(i)\), where \(\mathrm{cov}(i)\) denotes the set of evidence items supported by memory \(i\). This mapping is likewise used only for offline supervision and evaluation, not as an input to deployed policies. We define the realized coverage as:
\[
\mathrm{Cov}_t(A_t)=E_t \cap \bigcup_{i\in A_t}\mathrm{cov}(i).
\]

To capture both correctness and completeness of retention, we define token-weighted metrics:
\[
\begin{aligned}
T_t^{\mathrm{hit}}   &= \sum_{e \in E_t} \tau(e)\,\mathbf{1}[e \in \mathrm{Cov}_t(A_t)], \\
T_t^{\mathrm{miss}}  &= \sum_{e \in E_t} \tau(e)\,\mathbf{1}[e \notin \mathrm{Cov}_t(A_t)], \\
T_t^{\mathrm{reacq}} &= \sum_{e \in E_t \setminus \mathrm{Cov}_t(A_t)} \tau(e)\,\mathbf{1}[e \text{ is recoverable}], \\
T_t^{\mathrm{stale}} &= \sum_{i \in A_t} s_{t,i}\,\mathbf{1}[m_{t,i} \text{ is stale and used}],
\end{aligned}
\]
with a full coverage indicator $F_t=\mathbf{1}[E_t \subseteq \mathrm{Cov}_t(A_t)]$. The per-step reward is defined as:
\[
r_t := \alpha_{\mathrm{hit}}T_t^{\mathrm{hit}} + \alpha_{\mathrm{full}}F_t - \alpha_{\mathrm{store}}\sum_{i\in M_t} x_{t,i}s_{t,i}
- \alpha_{\mathrm{miss}}T_t^{\mathrm{miss}} - \alpha_{\mathrm{reacq}}T_t^{\mathrm{reacq}} - \alpha_{\mathrm{stale}}T_t^{\mathrm{stale}}.
\]
The objective is to find a policy $\pi$, mapping the observable state to a feasible retention action, that maximizes expected cumulative reward under the per-step budget constraint:
\[
\max_{\pi}\; \mathbb{E}_{\pi}\!\left[\sum_{t=1}^T r_t\right]
\quad\text{s.t.}\quad
\sum_{i\in M_t} x_{t,i}s_{t,i}\le B_t \quad \forall\, t\in\{1,\dots,T\},
\]
where the expectation is taken over the stochastic arrival of memories $N_t$ and queries $Q_t$, and the state evolves through the transition $C_t=A_t$ defined above, which links consecutive steps. Because the action at step $t$ determines the candidate pool $M_{t+1}=C_t\cup N_{t+1}$, the problem is a genuine sequential decision process rather than a sequence of independent subproblems.

\noindent\textbf{Multi-step structure and trade-offs.}
This carry-over makes the objective genuinely multi-step, and two features couple the steps. First, the hard budget forces selective retention, yet eviction costs are \emph{delayed}: a miss, reacquisition, or stale-use penalty is charged at the future step where the memory is actually needed, not at the step where it is dropped. Second, the action $A_t$ must be committed \emph{before} the evidence demand $E_t$ is revealed (Section~\ref{sec:observability}), so the policy cannot simply retain what each query needs and must instead anticipate future evidence demand, weighing short-term gains against long-term costs. A single-step optimizer, by contrast, maximizes only the immediate reward $r_t$ and may discard a memory that looks irrelevant now but becomes critical later. The deployed OSL-MR policy makes a state-dependent decision at each step---its candidate pool and history features depend on past actions---and is trained from sequential logs and evaluated on cumulative reward; it is thus an online approximation to the sequential optimum rather than an exact planner, whose closeness to the true optimum we quantify on small, exactly-solvable instances in Section~\ref{sec:optimality}.

\noindent\textbf{From intractability to learning-based approximation.}
This constrained partially observable stochastic optimization admits no closed-form or standard-solver solution: future evidence demands depend on unknown user queries, rewards involve delayed outcomes not observable at decision time, and the state space is combinatorial in the number of memories. We prove that even the single-step version is NP-hard via reduction from budgeted maximum coverage (Proposition~\ref{prop:nphard}). Because the multi-step problem contains the single-step case as a special instance (T=1), the full multi-step problem is NP-hard as well. Hardness persists under partial observability, which only restricts the policy's information without simplifying the underlying combinatorial selection. We therefore learn a retention policy from interaction logs. Since no logs exist initially, we first deploy a lightweight, budget-aware heuristic (Mixed-Score), intentionally stronger than naive baselines such as recency or Generative Agents, that keeps the system operational from the first query while logging all user-agent interactions. Once enough logs accumulate, we train an evidence learner offline under the same observability constraints to approximate the optimal policy. We additionally feed the Mixed-Score value to the learner as an extra input feature, which acts as an inductive prior that improves precision and modestly improves recall (ablation in Section~\ref{sec:experiments}). Mixed-Score thus plays two roles---cold-start deployment and a feature prior---while the learner remains the core performance driver, circumventing unobservability by exploiting offline evidence labels.

\noindent\textbf{Challenge of reward supervision.}
One might instead train a model directly on the per-step or cumulative reward. However, reward is a sequence-level scalar that cannot be decomposed into per-memory credit: a miss penalty occurs when a future query's evidence is missing, but which specific discarded memory caused it is ambiguous when several memories support the same evidence. Obtaining reward labels would moreover require evaluating every retained subset $A_t \subseteq M_t$---a combinatorial space of size $2^{|M_t|}$. Direct reward regression is therefore both ill-defined and computationally infeasible, so we introduce a decomposable proxy based on evidence membership.

\subsection{Mixed-Score Retention Policy}
\label{sec:mixed-score}

Mixed-Score is the foundational retention mechanism in OSL-MR. Rather than relying on recency or static importance, it integrates multiple online-observable signals into a unified utility function reflecting both relevance and resource cost. All components are constructed under strict observability constraints, ensuring direct deployability.

At each decision step \(t\), the agent observes a memory pool \(M_t\) of size \(n_t = |M_t|\). For each candidate memory utterance \(m_i\) with token size \(s_{t,i}\), Mixed-Score computes:
\[
\mathrm{MS}_{t,i} = \phi_{t,i} \cdot \biggl(
\begin{aligned}
& w_r \mathrm{Rec}_{t,i} + w_q \mathrm{Qov}_{t,i} + w_a \mathrm{Aov}_{t,i} \\
& \qquad + w_e \mathrm{Ent}_{t,i} + w_s \mathrm{Ses}_{t,i} + w_u \mathrm{Spk}_{t,i} - w_c \frac{s_{t,i}}{B_t}
\end{aligned}
\Biggr),\]
where each term is an online-observable utility signal: a recency score \(\mathrm{Rec}_{t,i}\), a question-overlap score \(\mathrm{Qov}_{t,i}\), an (offline-only, \(w_a{=}0\) at deployment) answer-overlap score \(\mathrm{Aov}_{t,i}\), an entity-overlap score \(\mathrm{Ent}_{t,i}\), session and speaker indicators \(\mathrm{Ses}_{t,i},\mathrm{Spk}_{t,i}\), and a size penalty \(s_{t,i}/B_t\); the non-negative weight vector \(\mathbf{w}=(w_r,w_q,w_a,w_e,w_s,w_u,w_c)\) is tuned on the development split via grid search. The precise definition of each term and the calibrated weights are given in the online supplement (Section~\ref{ec:mixed-score-features}).

The freshness multiplier \(\phi_{t,i}\in(0,1]\) is a continuous, online-computable proxy that downweights stale memories, so that outdated content contributes less to retention. Each item is assigned a binary type \(\mathrm{type}_i\in\{\)`\text{stable}', `\text{temporary}'\(\}\) at creation time (an offline LLM classifier or a deployment-time heuristic), and freshness decays with the age \(\Delta t_i = t - t_i^{\mathrm{store}}\) at a type-specific rate, \(\phi_{t,i} = \exp(-\gamma_i \Delta t_i)\), where the effective rate \(\gamma_i\) combines a type-specific base decay rate with a monotone time-scaling factor. Stable memories (e.g., long-lasting facts) decay slowly, whereas temporary memories (e.g., recent plans or transient intents) decay quickly, so a low \(\phi_{t,i}\) reduces the overall score and discourages retaining outdated memories. The corresponding stale-risk signal is \(\mathrm{stale\_risk}_{t,i} = 1 - \phi_{t,i}\). The exact decay equations, the type taxonomy, and the calibrated values are given in the online supplement (Section~\ref{app:freshness-params}).

To maximize total utility under the budget, the retained set \(A_t\) is obtained by sorting candidates in descending order of the score \(\mathrm{MS}_{t,i}\)---which already embeds the size penalty \(-w_c\,s_{t,i}/B_t\), so no further division by size is needed---and greedily adding them until the budget constraint \(\sum_{i \in A_t} s_{t,i} \le B_t\) is exhausted. Compared to prior heuristics such as recency-based decay or static importance scoring, Mixed-Score provides a more expressive, query-conditioned representation of memory utility. Because the answer-overlap term is disabled at deployment (\(w_a=0\)), the deployed scorer uses only online-observable signals and is directly deployable.

Beyond serving as a standalone, competitive baseline, Mixed-Score enables cold-start deployment before any interaction data exist, ensuring functionality from the first user query, and acts as an inductive prior that guides the learning-based model described next.

\subsection{Evidence Learning with Mixed-Score Priors}
\label{sec:evidence-learning}

While Mixed-Score provides a strong and fully deployable heuristic policy, it remains manually designed. As interaction data accumulate, OSL-MR transitions to a data-driven regime where a learned model refines this heuristic using evidence supervision. The learner is trained offline and deployed with frozen parameters, introducing no additional observability requirements beyond the online features.

\noindent\textbf{Model architecture and training.}
The evidence learner is a lightweight feedforward neural network \(f_\theta\) that takes per-memory feature vectors \(h_{t,i}\) as input and outputs a logit \(z_{t,i}=f_\theta(h_{t,i})\), which is then squashed by a sigmoid to a probability. The input features \(h_{t,i}\) consist of the online-observable features \(\psi_{t,i}^{on}\) (query context, memory metadata, interaction history, and the stale-risk signal \(\mathrm{stale\_risk}_{t,i}=1-\phi_{t,i}\) from Section~\ref{sec:mixed-score}) and optionally the Mixed-Score prior. The network is a multilayer perceptron (MLP) with a single hidden layer; the exact architecture and training hyperparameters are detailed in Section~\ref{sec:experiments}. Training minimizes the weighted binary cross-entropy loss:
\[
\mathcal{L}_{\mathrm{evid}}
=
\sum_{t,i}
\omega_{t,i}\,
\mathrm{BCE}
\bigl(
\sigma(f_\theta(h_{t,i})),
y_{t,i}^{\mathrm{evid}}
\bigr),
\]
where the evidence-support label is defined as $y_{t,i}^{\mathrm{evid}}
=
\mathbf{1}\!\left[
\mathrm{cov}_t(i)\cap E_t \neq \emptyset
\right]$. Here, \(\mathrm{cov}_t(i)\) denotes the set of evidence items supported by memory
\(i\) with respect to query \(Q_t\), \(E_t\) is the evidence demand set for
query \(Q_t\), \(\omega_{t,i}\) balances the positive and negative class
frequencies, and \(\sigma(\cdot)\) is the sigmoid function. %The coverage map
%\(\mathrm{cov}_t(i)\) is used only for offline supervision and is not available
%to the deployed policy.

\noindent\textbf{Why evidence membership is a valid proxy.}
The reward function contains a hit term \(\alpha_{\text{hit}}T_t^{\text{hit}}\) that directly sums over retained evidence, so a policy that retains evidence-relevant memories naturally increases the hit component. Moreover, all delayed penalties (miss, reacquisition, stale) are triggered precisely when a memory that would have been evidence is discarded or stale. Thus a model that accurately predicts evidence membership indirectly learns to maximize hits and avoid penalty-inducing conditions. While not strictly equivalent to maximizing cumulative reward (due to non-linear interactions and delays), this proxy provides a dense, decomposable, well-defined supervision signal that empirically yields high reward (Section~\ref{sec:experiments}), consistent with the alignment between reward rankings and evidence F1/precision observed in our experiments.

\noindent\textbf{Inductive prior from Mixed-Score.}
Instead of learning from scratch, the model can leverage Mixed-Score as an additional input feature. Two feature variants are considered:
\[
h_{t,i}^{\text{full}} = [\psi_{t,i}^{\text{on}}; \mathrm{MS}_{t,i}], \qquad
h_{t,i}^{\text{base}} = \psi_{t,i}^{\text{on}},
\]
where \(\mathrm{MS}_{t,i}\) is the Mixed-Score value (see Section~\ref{sec:mixed-score}). This lets the learner either rely purely on online signals or incorporate the heuristic as an inductive prior. Our ablation (Section~\ref{sec:experiments}) shows that adding the Mixed-Score prior consistently improves precision and modestly improves recall.

\noindent\textbf{Inference and budget-aware selection.}
At inference time, the trained network outputs a logit for each candidate memory. The retained set \(A_t\) is obtained via constrained selection:
\[
A_t = \mathrm{Select}\bigl(f_\theta(h_t), B_t\bigr),\quad \sum_{i\in A_t} s_{t,i} \le B_t,
\]
where \(\mathrm{Select}\) sorts memories by predicted probability (or logit) in descending order and adds them until the budget is exhausted, ensuring the final policy respects the budget while using only online-observable features.

\paragraph{Deployment protocol and answer-feature usage.}
Unless otherwise specified, online policies do not use answer-derived features. During cold start, the system runs a standalone Mixed-Score policy \(\mathrm{MS}^{0}\), whose parameters are set from domain expertise and whose answer-overlap weight is fixed to \(w_a=0\). This cold-start policy is used only to maintain stable system operation before sufficient interaction logs are available. After enough logs have been collected, Mixed-Score calibration is performed offline using logged queries, candidate memories, system-generated answers, and supervision or proxy labels; answer-derived features may be used only in this offline calibration stage. After calibration, we construct an answer-disabled calibrated prior, denoted by \(\mathrm{MS}^{cal,0}\), by setting the answer-overlap weight to \(w_a=0\). OSL-MR is then trained offline with inputs \(h_{t,i}=[\psi^{on}_{t,i};\mathrm{MS}^{cal,0}_{t,i}]\), where \(\psi^{on}\) contains only pre-answer online features. Once trained, \(f_\theta\) is frozen and deployed online with the same answer-free input form. Thus, the cold-start Mixed-Score policy, OSL-MR training features, and deployed OSL-MR policy all exclude answer-derived features; answer information and evidence labels are used only offline to calibrate the prior and construct supervision labels, never as inputs to the learned policy at inference. The overall training and deployment procedure is summarized in Algorithm~\ref{alg:framework} in the online supplement (Section~\ref{ec:algorithm}).

\subsection{Behavior-Cloning Variant}
\label{sec:bc}

For comparison, we construct a behavior-cloning baseline that learns from a non-deployable teacher policy built using offline-available supervision (OAS). For each training instance, starting from the Mixed-Score selection, the teacher performs a local greedy search: it considers swapping one retained memory with a currently omitted candidate, evaluates the resulting retained set using the reward of Section~\ref{sec:formulation} (which depends on gold evidence), and accepts the swap if it improves the reward, continuing until no improvement is found. The teacher thus produces a higher-quality retained set \(A_t^{teacher}\) than the original Mixed-Score, but relies on OAS and is therefore unsuitable for deployment.

The pseudo-labels for behavior cloning are defined as \(y_{t,i}^{BC} = \mathbf{1}[m_{t,i}\in A_t^{teacher}]\). The BC-learner is trained using the same architecture and loss as the evidence learner, but with these pseudo-labels instead of gold evidence labels, and uses the same online-observable features \(\psi_{t,i}^{on}\) (including the stale-risk signal \(\mathrm{stale\_risk}_{t,i}\)) as input. This variant evaluates whether imitating an improved heuristic policy can substitute for direct evidence supervision.

% The OSL-MR training-and-deployment algorithm has been moved to the online
% supplement (Section~\ref{ec:algorithm}, Algorithm~\ref{alg:framework}).
\section{Numerical Experiments}
\label{sec:experiments}

\subsection{Experimental Setup}
\label{sec:eval-protocol}

\noindent\textbf{Online/OAS discipline.}
All experiments strictly follow the observability separation defined in Section~\ref{sec:method}.
Deployable policies are restricted to online-observable inputs, including query context, memory metadata, recency, question overlap, entity overlap, session/speaker information, and cost features.
Offline-available supervision (OAS) includes gold evidence, answer text, realized coverage, miss events, reacquisition cost, and ground-truth freshness. OAS is used only for training or evaluation of oracle/teacher variants and is never accessible to deployable policies.
Any method that relies on OAS at inference time is treated as an oracle baseline, not a deployable policy.

\noindent\textbf{Benchmarks.}
We evaluate on two public long-horizon memory benchmarks: LoCoMo~\citep{maharana2024evaluatinglongtermconversationalmemory} and LongMemEval~\citep{wu2025longmemevalbenchmarkingchatassistants}. Both benchmarks provide multi-turn conversational interactions with gold evidence labels, enabling systematic evaluation of memory retention under budget constraints.

\noindent\textbf{Sequential evaluation protocol.}
We evaluate every method as a policy in the sequential decision process of Section~\ref{sec:method}, rather than on independent per-query instances. Each conversation forms an episode whose questions are processed in temporal order. The policy maintains a persistent retained store (cache) $C_t$ across steps: at step $t$ it observes the candidate pool $M_t = C_{t-1}\cup N_t$, consisting of the items still in the cache from the previous step ($C_{t-1}$) together with the newly arrived turns $N_t$ since the last decision; it then selects a retained subset $A_t\subseteq M_t$ under the storage budget $B_t$, and the cache is updated as $C_t = A_t$. A memory evicted at an earlier step therefore leaves the active pool and is no longer directly selectable, so a later query that requires it incurs a miss together with a reacquisition cost (the item is recoverable from the full history only at a cost), exactly as modeled by the reward $r_t$ in Section~\ref{sec:method}; per-memory freshness decays with elapsed steps. This carry-over makes the evaluation genuinely multi-step: a decision at step $t$ affects candidate availability and realized costs at all later steps $t'>t$. All compared policies---heuristic, behavior-cloning, and learned---are run under this identical sequential protocol and differ only in how they score and select memories at each step; reported metrics are aggregated over all steps of all test episodes. The per-step reward uses the calibrated coefficients listed below.

\noindent\textbf{Compared methods.}
We compare six approaches, described in full in the online supplement (Section~\ref{ec:baselines}):
\textbf{Recency}, which retains the most recent items that fit under the budget;
\textbf{Generative Agents (GA)} \citep{park2023generative}, scoring memories by recency, relevance, and importance, in two importance variants (GA-heuristic and GA-LLM);
\textbf{Mixed-Score}, our online-safe heuristic prior (Section~\ref{sec:mixed-score});
\textbf{BC-learner}, a behavior-cloning baseline that imitates a non-deployable OAS teacher (Section~\ref{sec:bc});
\textbf{OSL-MR (full)}, our complete evidence-learning method; and
\textbf{OSL-MR (w/o prior)}, an ablation that removes the Mixed-Score prior from the feature vector while keeping all other settings fixed.
All deployable policies use only online-observable features at inference, and every method---heuristic, behavior-cloning, and learned---is run under the identical sequential protocol above. We do not include the single-step optimizer of \citet{fofadiya2026forgetting} as a separate baseline (no public code); its best case is subsumed by the \emph{exact} single-step optimizer (Myopic) studied in Section~\ref{sec:optimality}, which already leaves a large optimality gap ($19\%$--$32\%$) once demand is temporally coupled, showing that the limitation is intrinsic to the single-step objective rather than to any one heuristic (see Section~\ref{ec:baselines}).

\noindent\textbf{Hyperparameters and cost settings.}
The evidence learner is a single-hidden-layer MLP (16 hidden units, sigmoid output). Default reward coefficients use storage cost as the unit baseline: $\alpha_{\mathrm{store}}=1.0$, $\alpha_{\mathrm{hit}}=4.0$, $\alpha_{\mathrm{reacq}}=6.0$, $\alpha_{\mathrm{miss}}=12.0$, $\alpha_{\mathrm{stale}}=6.0$, and $\alpha_{\mathrm{full}}=64.0$; the miss penalty is largest because failing to answer is most harmful, the reacquisition and stale penalties are moderate, and the full-coverage bonus strongly rewards complete evidence retention. Budgets reflect average evidence length: $32/64/128$ tokens for LoCoMo and $256/512/1024$ tokens for LongMemEval, covering under- and over-provisioned regimes. Full training settings, the offline LLM-annotation setup (used only to build the GA-LLM baseline and freshness proxies, never at evaluation), and the sensitivity-analysis grid are given in the online supplement (Section~\ref{ec:metrics}).

\noindent\textbf{Evaluation protocol.}
Learned policies (BC-learner and OSL-MR) are evaluated with five-fold cross-validation over episodes (four folds for training and validation, one held out for testing), repeated with three random seeds; unless otherwise noted, every table entry is the mean over held-out folds with the standard deviation in parentheses, using early stopping and L2 regularization as safeguards. The deterministic Mixed-Score heuristic is calibrated by grid search on the strictly separated training split with its answer-overlap weight fixed to $w_a=0$ (so it uses only online-observable features) and evaluated on the held-out test split; the same online-safe form serves as the cold-start policy and the learner's prior $\mathrm{MS}^{cal,0}$ (Section~\ref{sec:evidence-learning}). An information-favored variant that frees $w_a$ is reported only as a reference (Section~\ref{ec:prior-ablation}); full protocol details are in Section~\ref{ec:metrics}.

\noindent\textbf{Metrics.}
We report evidence precision, recall, F1, budget occupancy, and a shifted reward. Precision measures budget efficiency (capacity not wasted on irrelevant memories), recall measures evidence coverage, and F1 balances the two; budget occupancy is the fraction of the budget consumed by the retained subset and should be read together with precision and recall rather than as a stand-alone quality measure. The reward is the per-step reward of Section~\ref{sec:method} reported after a purely cosmetic global additive shift that does not affect any relative comparison. Formal definitions of all metrics and the shift constants are given in the online supplement (Section~\ref{ec:metrics}).

\subsection{Main Results}
\label{sec:main-results}

\noindent\textbf{The proposed method OSL-MR substantially outperforms heuristics.}
Across LoCoMo and LongMemEval, OSL-MR attains the best precision, recall, F1, and reward among all non-oracle methods (Recency, GA, Mixed-Score, and BC-learner) at every evaluated budget, and it does so in \emph{every} cross-validation fold, so its advantage is not an averaging artifact. For example, on LoCoMo at budget 128 (Table~\ref{tab:locomo-main-results}), OSL-MR reaches precision $0.201$, recall $0.561$, F1 $0.263$, and reward $246.209$ while consuming the least budget (occupancy $0.828$), whereas the strongest heuristic, Mixed-Score, trails on all four quality metrics (precision $0.051$, recall $0.401$, F1 $0.087$, reward $120.476$) yet runs at near-full occupancy ($0.988$), and Recency and the GA variants are weaker still (F1 below $0.03$, reward below $11$). OSL-MR also keeps occupancy below saturation throughout, and its lead over the best baseline (typically BC-learner) narrows as budgets loosen but never reverses: on LongMemEval the occupancy gap to BC-learner shrinks from $0.12$ at budget 256 to $0.04$ at budget 1024, while the heuristics' occupancy stays above $0.98$. Reward gains track the precision/recall and F1 ordering in every fold, confirming that the objective captures genuine retention quality rather than cost-calibration effects.

\noindent\textbf{GA performance and the importance--evidence mismatch.}
Both GA variants achieve low evidence F1 despite capturing plausible notions of memory salience. The gap is particularly instructive for GA-LLM: LLM‑prompted importance scores reflect general memory noteworthiness (e.g., an emotional event), but the evaluation metric measures whether the memory is gold evidence for the current query—a query‑conditional notion of utility. This suggests that static importance scores are structurally misaligned with query‑specific evidence needs. This mismatch adds further support for our evidence-supervised design: it strengthens the case for learning evidence membership directly from gold evidence labels, as OSL-MR does, rather than relying on an importance oracle.

\begin{table*}[t]
\centering
% \small
\footnotesize
\setlength{\tabcolsep}{4pt}  
\caption{LoCoMo results across memory budgets. Occupancy denotes budget utilization. Rewards are shifted by \(C = 476.136\) so that the worst baseline (GA-LLM at budget 128) is zero. Each entry is the mean over the cross-validation runs, with the standard deviation across runs in parentheses.}
\label{tab:locomo-main-results}
% mean with standard deviation in small parentheses
\newcommand{\val}[2]{#1\,\ensuremath{{\scriptstyle(\pm#2)}}}
\resizebox{\textwidth}{!}{%
\begin{tabular}{rl ccc c r}
\toprule
\textbf{Budget} & \textbf{Method} & \textbf{Precision} & \textbf{Recall} & \textbf{F1} & \textbf{Occupancy} & \textbf{Reward} \\
\midrule
32  & Recency      & \val{0.013}{0.009} & \val{0.032}{0.018} & \val{0.017}{0.011} & \val{0.953}{0.021} & \val{75.454}{58.158} \\
32  & GA-heuristic & \val{0.015}{0.007} & \val{0.035}{0.014} & \val{0.019}{0.009} & \val{0.948}{0.021} & \val{76.612}{57.608} \\
32  & GA-LLM       & \val{0.013}{0.007} & \val{0.028}{0.017} & \val{0.016}{0.009} & \val{0.947}{0.021} & \val{75.279}{56.474} \\
32  & Mixed-Score  & \val{0.043}{0.022} & \val{0.104}{0.036} & \val{0.057}{0.026} & \val{0.965}{0.016} & \val{94.729}{61.989} \\
32  & BC-learner   & \val{0.091}{0.012} & \val{0.107}{0.029} & \val{0.090}{0.017} & \val{0.734}{0.037} & \val{107.528}{57.888} \\
32  & OSL-MR       & \val{0.123}{0.023} & \val{0.143}{0.034} & \val{0.123}{0.027} & \val{0.665}{0.075} & \val{121.458}{57.386} \\
\midrule
64  & Recency      & \val{0.014}{0.008} & \val{0.063}{0.029} & \val{0.022}{0.011} & \val{0.971}{0.014} & \val{52.727}{60.776} \\
64  & GA-heuristic & \val{0.015}{0.007} & \val{0.075}{0.024} & \val{0.024}{0.010} & \val{0.971}{0.009} & \val{55.941}{60.675} \\
64  & GA-LLM       & \val{0.013}{0.007} & \val{0.059}{0.030} & \val{0.020}{0.011} & \val{0.971}{0.010} & \val{51.811}{61.207} \\
64  & Mixed-Score  & \val{0.049}{0.017} & \val{0.221}{0.052} & \val{0.077}{0.025} & \val{0.982}{0.008} & \val{104.971}{68.844} \\
64  & BC-learner   & \val{0.144}{0.040} & \val{0.257}{0.028} & \val{0.167}{0.027} & \val{0.835}{0.024} & \val{148.040}{41.633} \\
64  & OSL-MR       & \val{0.216}{0.058} & \val{0.347}{0.015} & \val{0.239}{0.038} & \val{0.794}{0.060} & \val{194.796}{38.486} \\
\midrule
128 & Recency      & \val{0.013}{0.006} & \val{0.098}{0.032} & \val{0.022}{0.009} & \val{0.983}{0.006} & \val{0.580}{62.698} \\
128 & GA-heuristic & \val{0.015}{0.004} & \val{0.130}{0.024} & \val{0.026}{0.007} & \val{0.980}{0.006} & \val{10.460}{59.630} \\
128 & GA-LLM       & \val{0.012}{0.004} & \val{0.100}{0.026} & \val{0.020}{0.007} & \val{0.981}{0.007} & \val{0.000}{61.080} \\
128 & Mixed-Score  & \val{0.051}{0.012} & \val{0.401}{0.052} & \val{0.087}{0.020} & \val{0.988}{0.005} & \val{120.476}{68.554} \\
128 & BC-learner   & \val{0.125}{0.028} & \val{0.460}{0.024} & \val{0.184}{0.029} & \val{0.900}{0.011} & \val{184.899}{26.024} \\
128 & OSL-MR       & \val{0.201}{0.067} & \val{0.561}{0.086} & \val{0.263}{0.068} & \val{0.828}{0.070} & \val{246.209}{90.323} \\
\bottomrule
\end{tabular}%
}
\end{table*}

\noindent\textbf{Direct evidence supervision outperforms behavior cloning.}
BC-learner improves over static heuristics but stays consistently weaker than OSL-MR. On LongMemEval at budget 256 (Table~\ref{tab:longmemeval}), OSL-MR leads BC-learner on all four quality metrics---precision $0.382$ versus $0.220$, recall $0.617$ versus $0.474$, F1 $0.401$ versus $0.258$, and reward $1347.546$ versus $1070.364$---while using less budget (occupancy $0.768$ versus $0.891$). OSL-MR keeps this advantage at the other budgets and on LoCoMo, confirming that gold evidence labels are a better‑aligned training target than teacher‑imitated pseudo‑labels; its reward gain reflects superior retention, not a cost trade‑off.

\begin{table}[t]
\centering
% \small
\footnotesize
\setlength{\tabcolsep}{4pt}  
\caption{LongMemEval results across memory budgets. Occupancy denotes budget utilization. Rewards are shifted by \(C = 5211.758\). Each entry is the mean over the cross-validation runs, with the standard deviation across runs in parentheses.}
\label{tab:longmemeval}
% mean with standard deviation in small parentheses
\newcommand{\val}[2]{#1\,\ensuremath{{\scriptstyle(\pm#2)}}}
\resizebox{\textwidth}{!}{%
\begin{tabular}{rl ccc c r}
\toprule
\textbf{Budget} & \textbf{Method} & \textbf{Precision} & \textbf{Recall} & \textbf{F1} & \textbf{Occupancy} & \textbf{Reward} \\
\midrule
256  & Recency      & \val{0.009}{0.005} & \val{0.054}{0.040} & \val{0.013}{0.007} & \val{0.988}{0.001} & \val{513.324}{1348.542} \\
256  & GA-heuristic & \val{0.014}{0.006} & \val{0.082}{0.061} & \val{0.020}{0.011} & \val{0.988}{0.001} & \val{537.194}{1333.780} \\
256  & GA-LLM       & \val{0.008}{0.004} & \val{0.051}{0.040} & \val{0.012}{0.006} & \val{0.985}{0.001} & \val{508.890}{1340.082} \\
256  & Mixed-Score  & \val{0.048}{0.010} & \val{0.319}{0.138} & \val{0.073}{0.018} & \val{0.993}{0.002} & \val{817.270}{1348.513} \\
256  & BC-learner   & \val{0.220}{0.067} & \val{0.474}{0.114} & \val{0.258}{0.074} & \val{0.891}{0.074} & \val{1070.364}{1400.738} \\
256  & OSL-MR       & \val{0.382}{0.125} & \val{0.617}{0.086} & \val{0.401}{0.093} & \val{0.768}{0.137} & \val{1347.546}{1331.998} \\
\midrule
512  & Recency      & \val{0.013}{0.006} & \val{0.120}{0.095} & \val{0.020}{0.010} & \val{0.992}{0.003} & \val{329.670}{1347.261} \\
512  & GA-heuristic & \val{0.017}{0.006} & \val{0.179}{0.106} & \val{0.028}{0.010} & \val{0.992}{0.003} & \val{390.298}{1343.121} \\
512  & GA-LLM       & \val{0.012}{0.004} & \val{0.129}{0.087} & \val{0.020}{0.008} & \val{0.989}{0.004} & \val{332.838}{1349.507} \\
512  & Mixed-Score  & \val{0.042}{0.008} & \val{0.452}{0.167} & \val{0.070}{0.015} & \val{0.995}{0.002} & \val{788.716}{1355.061} \\
512  & BC-learner   & \val{0.195}{0.073} & \val{0.627}{0.085} & \val{0.253}{0.082} & \val{0.874}{0.103} & \val{1158.614}{1421.444} \\
512  & OSL-MR       & \val{0.224}{0.066} & \val{0.739}{0.038} & \val{0.292}{0.061} & \val{0.870}{0.130} & \val{1362.376}{1429.565} \\
\midrule
1024 & Recency      & \val{0.016}{0.005} & \val{0.246}{0.143} & \val{0.027}{0.009} & \val{0.994}{0.004} & \val{0.000}{1353.018} \\
1024 & GA-heuristic & \val{0.019}{0.004} & \val{0.312}{0.137} & \val{0.032}{0.007} & \val{0.993}{0.005} & \val{63.702}{1358.589} \\
1024 & GA-LLM       & \val{0.016}{0.003} & \val{0.267}{0.138} & \val{0.027}{0.007} & \val{0.992}{0.005} & \val{5.648}{1377.175} \\
1024 & Mixed-Score  & \val{0.034}{0.005} & \val{0.542}{0.155} & \val{0.059}{0.010} & \val{0.995}{0.004} & \val{517.944}{1335.274} \\
1024 & BC-learner   & \val{0.187}{0.078} & \val{0.778}{0.059} & \val{0.257}{0.092} & \val{0.812}{0.147} & \val{1161.982}{1523.880} \\
1024 & OSL-MR       & \val{0.238}{0.094} & \val{0.809}{0.048} & \val{0.300}{0.072} & \val{0.769}{0.124} & \val{1386.078}{1646.711} \\
\bottomrule
\end{tabular}%
}
\end{table}

\noindent\textbf{Effect of the Mixed-Score prior.}
Removing the Mixed-Score prior from the learner's feature vector---the OSL-MR (w/o prior) ablation, which keeps the same online-safe evidence labels and architecture and differs only in the scalar prior term of \(h_{t,i}=[\psi_{t,i}^{on}; \mathrm{MS}_{t,i}]\)---consistently reduces precision and slightly reduces recall, and with them F1 and reward, while pushing budget occupancy up; the effect is consistent across all budgets and both benchmarks. For example, on LoCoMo at budget 128, precision falls from \(0.201\) to \(0.170\), recall from \(0.561\) to \(0.550\), and F1 from \(0.263\) to \(0.235\), while reward decreases from \(246.209\) to \(232.568\) and occupancy rises from \(0.828\) to \(0.868\). The advantage holds in every fold, confirming that the prior is a useful inductive bias for selection efficiency without adding supervision. The full ablation, together with the observability-leakage diagnostic that motivates the strict online/OAS separation, is reported in the online supplement (Section~\ref{ec:prior-ablation}, Table~\ref{tab:prior-ablation-combined}).

\subsection{Optimality Gap on Exactly-Solvable Instances}
\label{sec:optimality}

Our primary evaluation is on the LoCoMo and LongMemEval benchmarks (Section~\ref{sec:main-results}), where OSL-MR is the best-performing deployable method. Building on those results, we now ask a complementary and more refined question: in absolute terms, \emph{how close to optimal} is the learned policy, and how much does the \emph{multi-step} structure of the problem contribute to its advantage? These questions cannot be answered on the benchmarks themselves. Because retention is a genuinely multi-step problem (Section~\ref{sec:method}), the benchmark instances are long-horizon and the optimal policy is computationally intractable to obtain---the problem is NP-hard (Proposition~\ref{prop:nphard})---so there is no exact optimum on LoCoMo or LongMemEval to serve as a reference point. We therefore complement the main results with a controlled study on smaller, exactly-solvable instances, where the true optimum is available and the absolute quality of the learned policy can be measured directly.

Concretely, we construct small instances of the \emph{same} sequential decision process of Section~\ref{sec:method}---a per-step storage budget, evidence hits, misses served by a delayed and costly reacquisition channel, and a retained store carried over between steps ($C_t=A_t$)---cast as a finite Markov decision process (MDP) that is small enough to solve \emph{exactly} by dynamic programming (DP). On these instances we can compute both the true optimum $V^\star$ and the exact single-step optimum, which lets us analyze two things directly: (i) \emph{how much the multi-step structure matters}, by measuring how far even a \emph{perfect} single-step optimizer falls short of the optimum; and (ii) \emph{how near-optimal}, in absolute terms, our deployable learning-based policy is. Together these analyses provide direct, quantitative support for casting retention as a sequential problem and for approximating it by learning.

\noindent\textbf{Setup.}
Each instance has $n=6$ candidate memories, horizon $T=9$, and budget $B=3$, with a sticky candidate pool in which a dropped memory can re-enter only through the delayed, costly reacquisition channel of Section~\ref{sec:method}; per-step demand is either \emph{stationary} or \emph{shifts} between phases. We compare two groups of policies. \emph{Model-based references} know the MDP dynamics and are not deployable: \emph{Optimal (DP)}, the exact full-horizon optimum; and \emph{Lookahead-}$d$, a receding-horizon planner that runs the same DP truncated to a $d$-step window---$d{=}1$ is \emph{Myopic}, the exact single-step optimizer, and increasing $d$ ($2,4$) folds more of the future into the DP until it coincides with Optimal, isolating exactly how much pure planning depth is worth. \emph{Deployable policies} use only online-observable features and no planning: the \emph{Mixed-Score} greedy of Section~\ref{sec:mixed-score}, and \emph{OSL-MR (learned)}, a frozen counterpart of our evidence learner trained offline on a \emph{disjoint} set of instances under the online/OAS separation. For each policy $\pi$ we compute, exactly and without sampling, the expected cumulative return $V^\pi=\mathbb{E}_\pi[\sum_{t=1}^{T} r_t]$ and the optimality gap $\mathrm{gap}(\pi)=(V^\star-V^\pi)/|V^\star|$ relative to the DP optimum $V^\star=\max_\pi V^\pi$, each averaged over the $40$ test seeds (so the averaged \emph{Gap} need not equal the ratio of the averaged \emph{Return} values). The full instance generator, the rationale for this deliberately different comparison set, and the scope of the conclusions are detailed in the online supplement (Section~\ref{ec:optimality}). Measuring the gap in cumulative reward is the natural notion of optimality here, and because the reward ranking aligns with the evidence-quality metrics on the benchmarks (Section~\ref{sec:main-results}), a small reward gap also corresponds to near-best evidence retention.

\begin{table}[t]
\centering
\small
\caption{Exact optimality gaps on small, solvable instances of the retention MDP (Section~\ref{sec:method}), averaged over 40 held-out test seeds. Each of the three demand regimes occupies a pair of columns: \emph{Return} is the exact expected cumulative reward and \emph{Gap} the optimality gap $\,(V^\star-V^\pi)/|V^\star|\,$ to the optimum (lower is better; $0\%$ is optimal); each gap is computed per instance and then averaged over seeds, so it need not equal the ratio of the averaged \emph{Return} values. The methods form two groups: \emph{model-based references} (Optimal, Lookahead-$d$, Myopic) know the dynamics and are \emph{not} deployable, serving only to calibrate how much pure foresight is worth; \emph{deployable policies} (Mixed-Score, OSL-MR (learned)) use only online-observable features and no planning.}
\label{tab:synthetic-optimality}
\setlength{\tabcolsep}{5pt}
\begin{tabular}{l rr rr rr}
\toprule
 & \multicolumn{2}{c}{\textbf{Stationary}} & \multicolumn{2}{c}{\textbf{Shifting}} & \multicolumn{2}{c}{\textbf{Shifting + costly}} \\
\cmidrule(lr){2-3}\cmidrule(lr){4-5}\cmidrule(lr){6-7}
\textbf{Method} & \textbf{Return} & \textbf{Gap (\%)} & \textbf{Return} & \textbf{Gap (\%)} & \textbf{Return} & \textbf{Gap (\%)} \\
\midrule
\multicolumn{7}{l}{\textit{Model-based references (not deployable)}}\\
Optimal (DP)     & -12.00 & 0.0  & -11.13 & 0.0  & -27.03 & 0.0  \\
Lookahead-4      & -12.00 & 0.0  & -12.02 & 12.3 & -28.20 & 5.9  \\
Lookahead-2      & -12.00 & 0.0  & -12.22 & 14.6 & -28.69 & 7.5  \\
Myopic (1-step)  & -12.00 & 0.0  & -13.72 & 31.5 & -31.60 & 19.3 \\
\addlinespace
\multicolumn{7}{l}{\textit{Deployable (online features only)}}\\
Mixed-Score      & -12.26 & 2.0  & -12.09 & 11.3 & -28.64 & 6.6  \\
OSL-MR (learned) & -12.26 & 2.0  & -11.99 & 10.1 & -28.44 & 5.7  \\
\bottomrule
\end{tabular}
\end{table}

\noindent\textbf{A control: when the future does not matter, our approximation is essentially optimal.}
We first isolate the case in which the multi-step structure is absent. Under stationary demand the same memories are useful at every step, so retention decouples over time into a sequence of independent single-step problems. Here every method lands within about $2\%$ of the exact optimum: the planners are optimal, and both the deployable Mixed-Score greedy and the learned OSL-MR scorer attain a $2.0\%$ gap. This is exactly the behavior predicted by the submodular $(1-1/e)$ analysis (Remark~\ref{rem:greedy}), and it confirms that our greedy approximation loses essentially nothing when there is no temporal coupling to exploit. It also acts as a control for what follows: any larger gap we observe below must come from the multi-step structure itself, not from a weak heuristic.

\noindent\textbf{Multi-step matters: a perfect single-step optimizer provably fails.}
The situation changes markedly once demand shifts over time. A memory that appears irrelevant at the current step may be precisely what a later query needs, and once it has been dropped it can return only through the delayed, costly reacquisition channel. The optimal policy must therefore \emph{pre-load} such memories \emph{before} their demand rises---a decision whose benefit accrues only in the future and is therefore invisible to any per-step reward. The consequence is pronounced: the \emph{exact} single-step optimizer (Myopic), which maximizes the immediate reward perfectly at every step, is still $31.5\%$ away from the optimum under shifting demand and $19.3\%$ away when reacquisition is costly. Because Myopic is the \emph{best possible} one-step policy and not a weak baseline, this gap cannot be attributed to a poor heuristic---it is direct, quantitative evidence that retention does \emph{not} decompose into independent per-step problems. The multi-step formulation of Section~\ref{sec:method} is thus a necessity, not a modeling convenience.

\noindent\textbf{Learning to look ahead closes the gap, and does so best among deployable policies.}
What recovers this large gap is accounting for the future when scoring memories. Ranking memories by their \emph{long-run} rather than \emph{immediate} utility---exactly the inductive bias OSL-MR encodes through its demand-aware features and Mixed-Score prior---collapses the Myopic gap of $31.5\%/19.3\%$ down to roughly $6\%$--$11\%$. Among all deployable policies, the learned OSL-MR scorer comes closest to the optimum in every coupled regime: $10.1\%$ under shifting demand and $5.7\%$ under costly reacquisition, ahead of the hand-tuned Mixed-Score greedy ($11.3\%/6.6\%$) and on par with it when there is no temporal structure to exploit ($2.0\%$). Notably, OSL-MR matches---and slightly exceeds---explicit receding-horizon planning (Lookahead-4: $12.3\%/5.9\%$) while using \emph{no} planning and only online-observable features at inference; deeper lookahead narrows the gap only monotonically with depth and never overtakes the learned scorer. This indicates that the value of the multi-step view is realized not through expensive planning but by \emph{learning} which memories will matter later---precisely what OSL-MR is designed to do and precisely what a single-step objective cannot capture.

In short, the multi-step structure matters decisively---even the \emph{exact} single-step optimizer leaves a gap of roughly $19\%$--$32\%$ once demand shifts---while the deployable OSL-MR stays within about $2\%$ of the optimum when the future is irrelevant and about $6\%$--$10\%$ when it matters, the closest of all deployable policies and on par with model-based planners despite using \emph{no} planning. Because the optimum is unattainable at benchmark scale (NP-hard, Proposition~\ref{prop:nphard}), this controlled study supplies evidence the benchmarks cannot, justifying the learning-augmented approximation over an exact solver and corroborating OSL-MR's advantage on LoCoMo and LongMemEval.

\subsection{Robustness Analysis}
\label{sec:robustness}
We conduct a sensitivity analysis by varying one reward or cost coefficient at a time while keeping all other coefficients fixed at their default values.
The default coefficients are
$\alpha_{\mathrm{hit}}=4.0$,
$\alpha_{\mathrm{full}}=64.0$,
$\alpha_{\mathrm{miss}}=12.0$,
$\alpha_{\mathrm{reacq}}=6.0$,
and $\alpha_{\mathrm{stale}}=6.0$,
with the storage cost fixed to $\alpha_{\mathrm{store}}=1.0$.
For each coefficient, we apply a multiplier
$m \in \{0.25, 0.5, 0.75, 1.0, 1.5, 2.0, 4.0\}$
while leaving the others unchanged.
We report results under fixed memory budget 512 on LongMemEval here; results for other budgets (LoCoMo 32, 64, 128 and LongMemEval 256, 1024) are provided in the online supplement (Section~\ref{ec:robustness}).
As shown in the sensitivity curves (Fig.~\ref{fig:longmemeval-sensitivity_budget512}), OSL-MR remains the top-performing non-oracle method across all multipliers. For the reward coefficients $\alpha_{\mathrm{hit}}$ and $\alpha_{\mathrm{full}}$, its average reward increases or remains high as the corresponding reward becomes larger. For the penalty coefficients $\alpha_{\mathrm{miss}}$, $\alpha_{\mathrm{reacq}}$, and $\alpha_{\mathrm{stale}}$, the average reward decreases as expected, but OSL-MR consistently stays above the heuristic and learned baselines, showing robust performance under stronger cost penalties.

\begin{figure}[!t]
\centering
\includegraphics[width=\columnwidth]{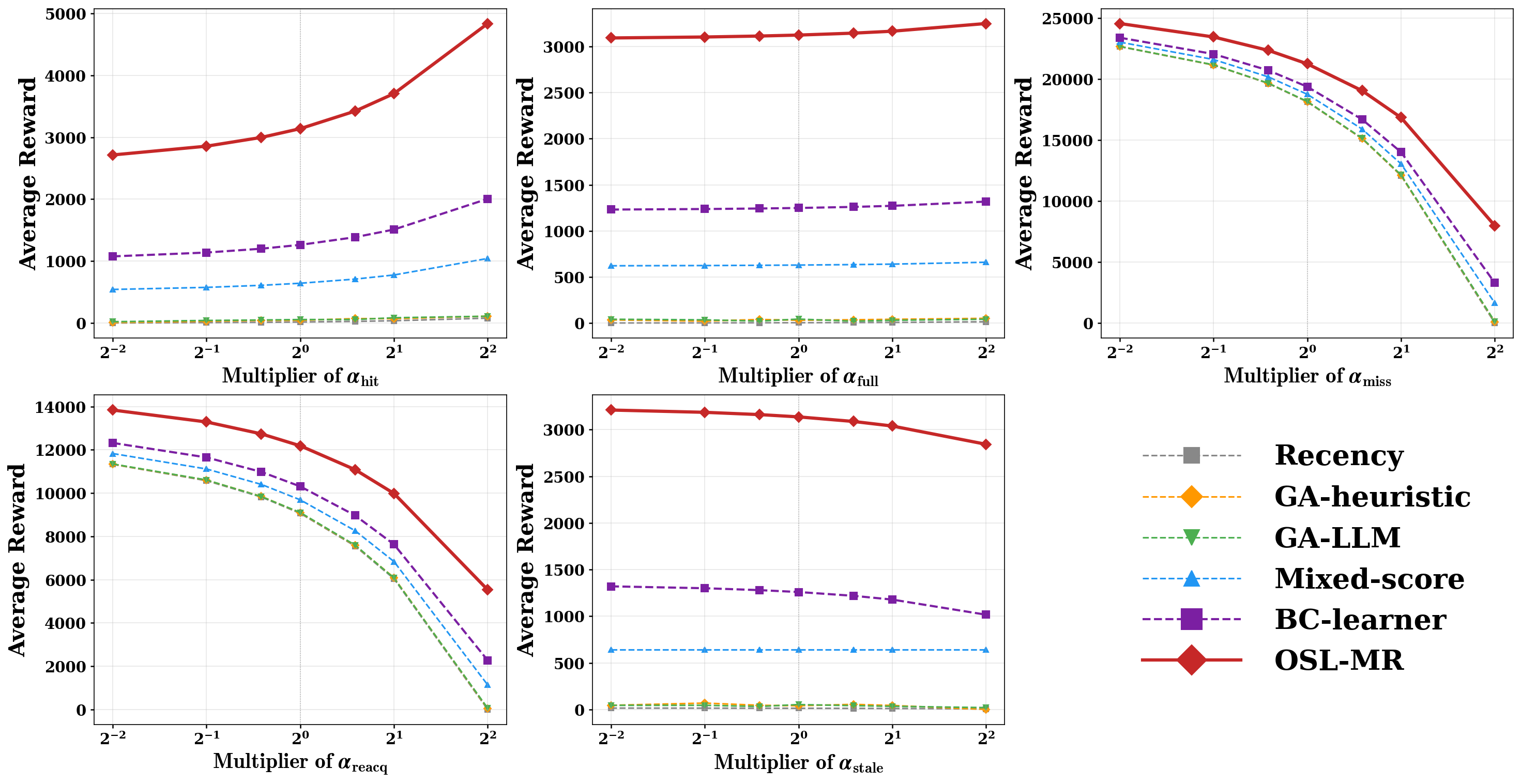}
\caption{Reward sensitivity on LongMemEval (budget=512) when varying a single coefficient. $\alpha_{\mathrm{store}}$ is fixed at $1.0$. Each curve changes one coefficient by multiplying its default value with $m \in \{0.25,0.5,0.75,1.0,1.5,2.0,4.0\}$, keeping all other coefficients at their defaults.}
\label{fig:longmemeval-sensitivity_budget512}
\end{figure}

\subsection{Summary of Findings}

In summary: (i) heuristics (Recency, Generative Agents, Mixed-Score) saturate the budget with low-utility items and attain poor evidence F1 under tight constraints; (ii) evidence-supervised learning is the primary driver, with OSL-MR consistently best across budgets and datasets and the largest gains under tight budgets; (iii) the Mixed-Score prior improves precision and modestly improves recall while reducing occupancy; (iv) direct evidence supervision outperforms behavior cloning; (v) reward rankings align with F1 and precision for all methods of practical interest, confirming that the objective captures genuine retention quality; and (vi) on exactly-solvable instances of the retention MDP, the \emph{exact} single-step optimizer leaves a substantial gap under shifting demand whereas OSL-MR stays close to the dynamic-programming optimum, evidencing the necessity of the sequential formulation and the near-optimality of our learning-based approximation. Together these findings validate OSL-MR as an online-safe evidence-learning framework for constrained memory retention.
\section{Conclusion and Discussion}
\label{sec:conclusion}

We have presented OSL-MR, a framework that treats memory retention in long-horizon language agents as a sequential decision problem under budget constraints, delayed feedback, and partial observability. Its central contribution is the formulation: to our knowledge, this is the first work to cast memory retention as a constrained stochastic optimization problem accounting for evidence utility, reacquisition cost, and stale risk, under a strict separation between online-observable features and offline-available supervision. Because the problem is NP-hard (Proposition~\ref{prop:nphard}) and partially observable, exact optimization is intractable, so we design a learning-based approximation that pairs a lightweight Mixed-Score heuristic (cold-start baseline and inductive prior) with an evidence learner trained from natural user-agent logs. On small, exactly-solvable instances of the same problem, this horizon-aware approximation is near-optimal under stationary demand and closes most of the optimality gap left by exact single-step optimization once demand shifts, achieving this through long-run utility scoring rather than expensive planning. Empirically, the learned policy consistently outperforms strong heuristics (recency, Generative Agents, and Mixed-Score itself) and behavior-cloning baselines on LoCoMo and LongMemEval, especially under tight budgets, with the Mixed-Score prior improving precision and modestly improving recall, and direct evidence supervision as the primary driver. The strict observability separation ensures that all reported gains reflect generalization under realistic information constraints.

Our framework complements recent advances in memory writing (Mem0), compression (MEM1), and retrieval-side optimization (MemRL): while those target other stages of the memory lifecycle, OSL-MR is the first to formalize and optimize retention under hard budgets and observability constraints, bridging constrained stochastic optimization and learning-based approximation. Future work includes automatically learning the Mixed-Score feature groups and weights, more principled freshness estimation beyond type-based decay, and scaling to embodied or tool-augmented agents via more efficient or amortized decision strategies.

% References here (outcomment the appropriate case)

% CASE 1: BiBTeX used to constantly update the references
%   (while the paper is being written).
%\bibliographystyle{informs2014} % outcomment this and next line in Case 1
%\bibliography{<your bib file(s)>} % if more than one, comma separated

%\bibliographystyle{informs2014} % outcomment this and next line in Case 1
%\bibliography{sample} % if more than one, comma separated

% CASE 2: BiBTeX used to generate mypaper.bbl (to be further fine tuned)
%\input{mypaper.bbl} % outcomment this line in Case 2

%If you don't use BiBTex, you can manually itemize references as shown below.

%\bibliographystyle{nonumber}
\bibliographystyle{informs2014}
\bibliography{06_references}

%%%%%%%%%%%%%%%%%%%%%%%%%%%%%%%%%%%%%%%%%%%%%%%%%%%%%%%%%%%%%%%%%%%%%%
% ONLINE SUPPLEMENT (E-COMPANION)
% Everything below \ECSwitch is the e-companion: it is page-numbered
% ec1, ec2, ... with figures and tables numbered EC.1, EC.2, ...
% The main paper (everything above) stays within the 25-page limit;
% for final submission this material is packaged as a separate
% online-supplement file, per IJOC submission guidelines.
%%%%%%%%%%%%%%%%%%%%%%%%%%%%%%%%%%%%%%%%%%%%%%%%%%%%%%%%%%%%%%%%%%%%%%
\ECSwitch
% Switch between the two e-companion versions here:
%   08_ecompanion             -> original, ~17 pages (kept for cross-checking)
%   08_ecompanion_compressed  -> compressed, <=10 pages (submission target)
% =====================================================================
%  ONLINE SUPPLEMENT (E-COMPANION) -- COMPRESSED VERSION (target <= 10 pages)
%  This is a length-reduced counterpart of 08_ecompanion.tex. It preserves
%  ALL labels, tables, the algorithm, the NP-hardness proof, and every
%  numerical value; only prose is tightened and the five sensitivity
%  figures are combined into one subfigure grid. The original 17-page
%  08_ecompanion.tex is kept unchanged for cross-checking.
% =====================================================================

\ECHead{E-Companion}

% Slightly reduced body font for the online supplement to keep it within the
% page guideline; section headings keep their normal size. This applies to the
% rest of the document (the e-companion is the final included file).
\small

% ---------------------------------------------------------------------
\section{Evaluation Metrics, Protocol, and Training Details}
\label{ec:metrics}

This section gives the full definitions of the evaluation metrics, the cross-validation protocol, the training hyperparameters, and the offline LLM-annotation setup summarized in Section~\ref{sec:eval-protocol}.

\paragraph{Shifted reward.}
We report the per-step reward of Section~\ref{sec:method} averaged over an episode's queries, after a global shift for readability: $\mathrm{shifted\_reward} = \mathrm{raw\_reward} + C$, where $C = 476.136$ on LoCoMo and $C = 5211.758$ on LongMemEval. The shift is purely cosmetic and does not affect any relative comparison; the minimum shifted reward in each benchmark is zero, corresponding to the worst-performing method under that configuration.

\paragraph{Evidence precision, recall, and F1.}
For each query $t$, let $G_t$ be the set of gold evidence dialogue IDs and $R_t$ the set of retained dialogue IDs under the budget. We compute
\[
\mathrm{Precision}_t = \frac{|R_t \cap G_t|}{\max(|R_t|,1)},
\qquad
\mathrm{Recall}_t = \frac{|R_t \cap G_t|}{\max(|G_t|,1)},
\qquad
\mathrm{F1}_t = \frac{2\,\mathrm{Precision}_t\,\mathrm{Recall}_t}{\mathrm{Precision}_t+\mathrm{Recall}_t},
\]
with $\mathrm{F1}_t=0$ when $\mathrm{Precision}_t+\mathrm{Recall}_t=0$, and all metrics macro-averaged over queries. Precision measures budget efficiency: high precision means the policy is not wasting capacity on irrelevant memories. Recall measures evidence coverage: high recall means most required information is retained. In budgeted retention, these two objectives are in tension, and F1 captures their balance. These set-level metrics aggregate over dialogue IDs, whereas the reward of Section~\ref{sec:method} aggregates token-weighted coverage through $\mathrm{cov}(i)$; the two are \emph{empirically} aligned (Section~\ref{sec:main-results}) rather than identical by construction, and the low absolute precision reflects that gold evidence occupies only a small fraction of the budgeted IDs.

\paragraph{Budget occupancy.}
Budget occupancy measures the fraction of the available memory budget consumed by the retained subset, $\mathrm{Occupancy}_t = \frac{\sum_{i\in A_t} s_{t,i}}{\max(B_t,1)}$, macro-averaged over queries. High occupancy is common for static scoring rules that greedily fill the budget, but does not necessarily indicate better retention quality. Learned policies that stop early when no remaining candidate appears useful may yield lower occupancy while achieving higher precision and reward. Occupancy should therefore be interpreted together with precision and recall, as a measure of budget efficiency rather than of evidence correctness.

\paragraph{Cross-validation protocol.}
To ensure a fair comparison and avoid overfitting, we adopt different evaluation protocols for the learned policies and the heuristic baseline. For learned policies (BC-learner and OSL-MR), we split episodes into five folds: in each run, four folds are used for training and validation and the held-out fold is used for testing, with early stopping on a validation subset drawn only from the training folds. We repeat the process with three random seeds and report the mean and standard deviation over all held-out-fold evaluations. The heuristic Mixed-Score has only a small number of hyperparameters (feature weights, decay rates, etc.), which are calibrated on the training set by grid search without any cross-validation, because the heuristic is deterministic and does not fit to noise; the calibrated parameters are then evaluated on the test set. For the learned policies we apply strict safeguards---cross-validation for hyperparameter selection, early stopping, and $\ell_2$ regularization. Unless otherwise noted, every entry in the results tables is the mean over these runs, and the value in parentheses is the corresponding standard deviation across runs.

\paragraph{Training hyperparameters and cost-coefficient rationale.}
The evidence learner is a single-hidden-layer MLP (16 hidden units, sigmoid output) trained for 12 epochs (learning rate $0.03$, $\ell_2$ regularization $10^{-4}$, positive class weight 4.0, negative down-sampling multiplier 6). Default reward coefficients use storage cost as the unit baseline: $\alpha_{\mathrm{store}}=1.0$, $\alpha_{\mathrm{hit}}=4.0$, $\alpha_{\mathrm{reacq}}=6.0$, $\alpha_{\mathrm{miss}}=12.0$, $\alpha_{\mathrm{stale}}=6.0$, $\alpha_{\mathrm{full}}=64.0$. The miss penalty is the highest because failing to answer is the most harmful outcome; the reacquisition and stale penalties are moderate, reflecting the cost of re-searching and the risk of outdated information; the hit reward is lower; and the full-coverage bonus strongly encourages complete evidence retention. Budgets reflect average evidence length: for LoCoMo ($\approx$60 tokens per query) we use 32, 64, and 128 tokens, and for LongMemEval ($\approx$587 tokens per query) we use 256, 512, and 1024 tokens, covering both under- and over-provisioned regimes. The sensitivity analysis (Section~\ref{sec:robustness}) varies each coefficient individually over a multiplier $m \in \{0.25,0.5,0.75,1.0,1.5,2.0,4.0\}$, confirming that the learning-based approximation is robust to the coefficient scales. These coefficients govern the full sequential model and the optimality study.

\paragraph{Offline LLM annotation.}
We use LLMs only for offline annotation (DeepSeek-V4-Flash, temperature 0.0, JSON prompts, three-turn context, with output limits of 20 tokens for memory type and 24 for importance). The importance scores provide the salience signal for the GA-LLM baseline, while the memory-type labels are used to derive the freshness proxies described above. All annotations are cached in advance, so no LLM calls are made during evaluation. The full prompt templates are provided in Section~\ref{app:llm-prompts}.

% ---------------------------------------------------------------------
\section{Detailed Descriptions of Compared Methods}
\label{ec:baselines}

\begin{itemize}\setlength\itemsep{2pt}
    \item \textbf{Recency}: a practical and widely used heuristic that ranks memories by most-recent access (ties broken by item value) and greedily admits them in that order under the size-capacity budget $B$, skipping any item that does not fit. It is simple, computationally efficient, and commonly employed as a baseline.
    \item \textbf{Generative Agents (GA)}: a widely adopted retrieval paradigm that scores each memory as a weighted sum of recency, relevance, and importance \citep{park2023generative}. We evaluate two importance variants: GA-heuristic (hand-crafted rules based on memory content, e.g., higher weight for emotional events) and GA-LLM (importance scores queried from an LLM offline); both variants share the same recency and relevance computation.
    \item \textbf{Mixed-Score}: our online-safe heuristic prior. It computes a utility score for each memory by combining multiple observable feature groups (temporal, relevance, context, risk) with a size penalty term, and selects the top-scoring memories under the budget. Its answer-overlap weight is fixed to $w_a=0$, so it relies only on online-observable signals and is fully deployable; this same online-safe form serves as the cold-start policy ($\mathrm{MS}^0$) and as the learner's inductive prior ($\mathrm{MS}^{cal,0}$). An information-favored variant that frees $w_a$ at test is reported only as a reference (Section~\ref{ec:prior-ablation}, Table~\ref{tab:mixed-wa-reference}).
    \item \textbf{BC-learner}: a behavior-cloning baseline trained to imitate the retention decisions of a non-deployable teacher policy. The teacher performs local greedy search around the Mixed-Score selection using oracle gold evidence (OAS) to produce higher-quality retained sets. BC-learner uses the same architecture as OSL-MR but is supervised by teacher pseudo-labels instead of gold evidence.
    \item \textbf{OSL-MR (full)}: our complete method, which trains an evidence learner from logged interaction data using gold evidence labels. The learner takes online-observable features concatenated with the Mixed-Score prior as input and outputs a utility logit per memory, and a budget-constrained decoder selects the retained subset at inference. The learned policy is frozen after training and uses no OAS at test time.
    \item \textbf{OSL-MR (w/o prior)}: an ablation variant that removes the Mixed-Score prior from the feature vector, using only online-observable features $\psi^{on}$. All other settings (training data, architecture, loss) remain identical, isolating the contribution of the prior.
\end{itemize}

\noindent\textbf{Note on single-step methods.}
Fofadiya and Tiwari~\citep{fofadiya2026forgetting} formulate single-step budgeted optimization that maximizes immediate relevance per step, but they ignore delayed consequences (miss, reacquisition, staleness) and partial observability, and their code is not publicly available; hence we do not include them as a quantitative baseline. Instead, our experiments include strong local heuristics (Recency, GA, Mixed-Score) that act as myopic, per-step decision policies; crucially, they are evaluated under the same sequential protocol (Section~\ref{sec:eval-protocol}), so the comparison isolates the value of horizon-aware retention while keeping the evaluation dynamics identical across methods. In addition, we do not rely on these heuristics alone to characterize single-step optimization: in the controlled study of Section~\ref{sec:optimality} we evaluate the \emph{exact} single-step optimizer (Myopic)---the strongest policy any single-step method can attain---and find that it leaves a large optimality gap ($19\%$--$32\%$) once demand is temporally coupled. This indicates that the limitation of single-step retention is intrinsic to the objective itself rather than specific to any one heuristic or implementation.

% ---------------------------------------------------------------------
\section{Mixed-Score Feature Definitions and Calibration}
\label{ec:mixed-score-features}

The terms of the Mixed-Score utility $\mathrm{MS}_{t,i}$ in Section~\ref{sec:mixed-score} are:
\begin{itemize}\setlength\itemsep{1pt}
    \item \(\mathrm{Rec}_{t,i} := \frac{t_i^{\mathrm{store}} + 1}{t + 1}\): recency, where \(t_i^{\mathrm{store}}\) is the creation step (indices from 0) and \(t\) the current step, so a memory created now (\(t_i^{\mathrm{store}}=t\)) scores $1$ and older ones less.
    \item \(\mathrm{Qov}_{t,i} := \frac{|T_{q,t} \cap T_{m_i}|}{|T_{q,t}|}\): question overlap, the token overlap between the current question \(T_{q,t}\) and the memory \(T_{m_i}\).
    \item \(\mathrm{Aov}_{t,i} := \frac{|T_{a,t} \cap T_{m_i}|}{|T_{a,t}|}\): overlap between memory and gold answer. It requires the answer and is therefore an offline-available supervision (OAS) signal; every online policy in this paper---\(\mathrm{MS}^0\), the prior \(\mathrm{MS}^{cal,0}\), the Mixed-Score baseline in the main results, and OSL-MR---sets \(w_a=0\) and uses only online-observable signals (Section~\ref{sec:evidence-learning}). An information-favored variant that frees \(w_a\) at test (hence not deployable) is reported only as a reference in Section~\ref{ec:prior-ablation}.
    \item \(\mathrm{Ent}_{t,i} := \frac{|E_{q,t} \cap E_{m_i}|}{|E_{q,t}|}\): entity overlap between question entity spans \(E_{q,t}\) and memory entity spans \(E_{m_i}\).
    \item \(\mathrm{Ses}_{t,i} := \mathbf{1}[\text{session}(m_i) = \max_{j \in M_t} \text{session}(m_j)]\): indicator that the memory is from the most recent session.
    \item \(\mathrm{Spk}_{t,i} := \mathbf{1}[\text{speaker}(m_i) = \text{user}]\): indicator that the memory was generated by the user.
    \item Budget \(B_t\) is the token budget and \(s_{t,i}\) the token size of memory $i$; \(\mathbf{w} = (w_r,w_q,w_a,w_e,w_s,w_u,w_c)\) are non-negative coefficients tuned on the development split (Section~\ref{app:freshness-params}).
\end{itemize}

% ---------------------------------------------------------------------
\section{OSL-MR Training and Deployment Algorithm}
\label{ec:algorithm}

Algorithm~\ref{alg:framework} summarizes the cold-start logging, offline calibration and supervision, and frozen-deployment stages of OSL-MR (Section~\ref{sec:evidence-learning}).

\begin{algorithm}[t]
\footnotesize
\caption{OSL-MR: Training and Frozen Deployment}
\label{alg:framework}
\begin{algorithmic}[1]
\Require \begin{minipage}[t]{0.86\linewidth}
Training environments \(\mathcal{E}\), budgets \(B_t\), online features \(\psi^{on}\), cold-start Mixed-Score \(\mathrm{MS}^{0}\) with \(w_a=0\),\\\ selection function \(\mathrm{Select}\)
\end{minipage}
\vspace{7pt}
\Ensure Frozen learned policy \(f_\theta\)
\State \(\mathcal{D}\gets\emptyset\)

\Statex \textbf{Cold-start logging.}
\ForAll{episode \(\in\mathcal{E}\), step \(t\)}
    \State Observe \(Q_t,M_t\)
    \State Compute \(\mathrm{MS}^{0}_{t,i}\) for all \(m_{t,i}\in M_t\) \hfill // \(w_a=0\); no answer feature online
    \State Execute \(A_t=\mathrm{Select}(\mathrm{MS}^{0}_t,B_t)\) and log \((Q_t,M_t,\psi^{on}_t)\)
\EndFor

\Statex \textbf{Offline calibration and supervision.}
\State Calibrate Mixed-Score offline  \hfill // answer features allowed only offline
\State Construct \(\mathrm{MS}^{cal,0}\) by setting \(w_a=0\)
\ForAll{logged \((Q_t,M_t,\psi^{on}_t)\)}
    \State Compute \(\mathrm{MS}^{cal,0}_{t,i}\) for all \(m_{t,i}\in M_t\) \hfill // no answer feature used in learner input
    \State Set \(h_{t,i}=[\psi^{on}_{t,i};\mathrm{MS}^{cal,0}_{t,i}]\)
    \State Obtain \(E_t\) and \(\mathrm{cov}(i)\) for all \(m_{t,i}\in M_t\) \hfill // offline/OAS only
    \State Get \(y^{\mathrm{evid}}_{t,i}
    =
    \mathbf{1}\!\left[\mathrm{cov}(i)\cap E_t\neq\emptyset\right]\) \hfill // evidence-support label
    \State \(\mathcal{D}\gets\mathcal{D}\cup\{(h_t,y^{\mathrm{evid}}_t)\}\)
\EndFor

\Statex \textbf{Training and frozen deployment.}
\State Train \(f_\theta\) on \(\mathcal{D}\) by solving \(\min_\theta \sum_{t,i}\omega_{t,i}\,\mathrm{BCE}\bigl(\sigma(f_\theta(h_{t,i})),y_{t,i}^{\mathrm{evid}}\bigr)\)
\State Freeze \(f_\theta\); deploy with \(h_{t,i}=[\psi^{on}_{t,i};\mathrm{MS}^{cal,0}_{t,i}]\)\hfill // no answer feature; no online update
\State Select \(A_t=\mathrm{Select}(f_\theta(h_t),B_t)\)

\State \Return \(f_\theta\)
\end{algorithmic}
\end{algorithm}

% ---------------------------------------------------------------------
\section{Mixed-Score Prior Ablation}
\label{ec:prior-ablation}

To isolate the prior's contribution, we compare full OSL-MR against a variant that removes the Mixed-Score prior (OSL-MR (w/o prior)); both use the same online-safe evidence labels and architecture, differing only in the scalar prior in \(h_{t,i}=[\psi_{t,i}^{on}; \mathrm{MS}_{t,i}]\). Removing the prior consistently lowers precision and slightly lowers recall, and with them F1 and reward, while raising budget occupancy across all budgets and both benchmarks (Table~\ref{tab:prior-ablation-combined}). For example, on LongMemEval at budget 256, precision falls from $0.382$ to $0.359$, recall from $0.617$ to $0.608$, and F1 from $0.401$ to $0.382$, while reward decreases from $1347.546$ to $1317.028$ and occupancy rises from $0.768$ to $0.794$. Full OSL-MR has higher precision and F1 in each of the five folds; we report means and omit per-run deviations.

\begin{table}[htbp]
\centering
\footnotesize
\caption{Mixed-score prior ablation across benchmarks. Rewards are shifted by $C = 476.136$ (LoCoMo) and $C = 5211.758$ (LongMemEval). Each entry is the mean over the cross-validation runs.}
\label{tab:prior-ablation-combined}
\newcommand{\val}[2]{#1}
\begin{tabular}{rlrrrrr}
\toprule
\textbf{Dataset} & \textbf{Method} & \textbf{Precision} & \textbf{Recall} & \textbf{F1} & \textbf{Occupancy} & \textbf{Reward} \\
\midrule
\multirow{11}{*}{LoCoMo} 
& \multicolumn{5}{c}{\textit{Budget 32}} \\
& OSL-MR 
& \val{0.123}{0.023} & \val{0.143}{0.034} & \val{0.123}{0.027} & \val{0.665}{0.075} & \val{121.458}{57.386} \\
& OSL-MR (w/o prior) 
& \val{0.096}{0.020} & \val{0.134}{0.038} & \val{0.103}{0.025} & \val{0.838}{0.040} & \val{112.515}{59.790} \\[4pt]
& \multicolumn{5}{c}{\textit{Budget 64}} \\
& OSL-MR 
& \val{0.216}{0.058} & \val{0.347}{0.015} & \val{0.239}{0.038} & \val{0.794}{0.060} & \val{194.796}{38.486} \\
& OSL-MR (w/o prior) 
& \val{0.186}{0.059} & \val{0.333}{0.022} & \val{0.213}{0.040} & \val{0.837}{0.068} & \val{182.823}{39.750} \\[4pt]
& \multicolumn{5}{c}{\textit{Budget 128}} \\
& OSL-MR 
& \val{0.201}{0.067} & \val{0.561}{0.086} & \val{0.263}{0.068} & \val{0.828}{0.070} & \val{246.209}{90.323} \\
& OSL-MR (w/o prior) 
& \val{0.170}{0.067} & \val{0.550}{0.090} & \val{0.235}{0.074} & \val{0.868}{0.074} & \val{232.568}{97.879} \\
\midrule
\multirow{11}{*}{LongMemEval} 
& \multicolumn{5}{c}{\textit{Budget 256}} \\
& OSL-MR 
& \val{0.382}{0.125} & \val{0.617}{0.086} & \val{0.401}{0.093} & \val{0.768}{0.137} & \val{1347.546}{1331.998} \\
& OSL-MR (w/o prior) 
& \val{0.359}{0.120} & \val{0.608}{0.087} & \val{0.382}{0.092} & \val{0.794}{0.134} & \val{1317.028}{1330.361} \\[4pt]
& \multicolumn{5}{c}{\textit{Budget 512}} \\
& OSL-MR 
& \val{0.224}{0.066} & \val{0.739}{0.038} & \val{0.292}{0.061} & \val{0.870}{0.130} & \val{1362.376}{1429.565} \\
& OSL-MR (w/o prior) 
& \val{0.211}{0.059} & \val{0.731}{0.041} & \val{0.280}{0.060} & \val{0.890}{0.112} & \val{1338.750}{1425.285} \\[4pt]
& \multicolumn{5}{c}{\textit{Budget 1024}} \\
& OSL-MR 
& \val{0.238}{0.094} & \val{0.809}{0.048} & \val{0.300}{0.072} & \val{0.769}{0.124} & \val{1386.078}{1646.711} \\
& OSL-MR (w/o prior) 
& \val{0.222}{0.101} & \val{0.798}{0.045} & \val{0.280}{0.076} & \val{0.802}{0.133} & \val{1301.674}{1604.608} \\
\bottomrule
\end{tabular}
\end{table}

\paragraph{Information-favored Mixed-Score reference ($w_a\neq0$).}
The deployable Mixed-Score baseline fixes $w_a=0$ and uses only online features. To probe how much an answer-overlap signal could help a strong heuristic, we also evaluate an \emph{information-favored} variant that frees $w_a$ during offline calibration on the training split and retains the calibrated (possibly nonzero) weight at test; computing the answer-overlap feature then requires the gold answer, so this variant is \emph{not} deployable and is reported only as a reference. Table~\ref{tab:mixed-wa-reference} contrasts the two. The answer-overlap signal yields only a small gain: on LoCoMo at budget 128, it improves F1 from $0.087$ to $0.095$, and the shifted reward from $120.476$ to $137.964$; on LongMemEval at budget 1024, the corresponding gains are F1 from $0.059$ to $0.061$, and shifted reward from $517.944$ to $570.358$. Even so, it remains far below OSL-MR, whose F1 reaches $0.263$ on LoCoMo at budget 128, and $0.300$ on LongMemEval at budget 1024, while using no answer information. The main-paper conclusions are thus conservative: OSL-MR surpasses the heuristic even when the latter is granted an offline-calibrated answer-overlap advantage unavailable to any deployed system.

\begin{table}[htbp]
\centering
\footnotesize
\caption{Information-favored Mixed-Score reference. ``Deployable ($w_a{=}0$)'' is the online-safe Mixed-Score used in the main results; ``Info-favored ($w_a{\neq}0$)'' frees the answer-overlap weight during offline calibration and retains it at test (requiring the gold answer, hence not deployable). Rewards are shifted by $C=476.136$ (LoCoMo) and $C=5211.758$ (LongMemEval), matching the main tables. Mean over cross-validation folds (std in parentheses).}
\label{tab:mixed-wa-reference}
\newcommand{\val}[2]{#1\,\ensuremath{{\scriptstyle(\pm#2)}}}
\begin{tabular}{rlrrrrr}
\toprule
\textbf{Budget} & \textbf{Mixed-Score} & \textbf{Precision} & \textbf{Recall} & \textbf{F1} & \textbf{Occupancy} & \textbf{Reward} \\
\midrule
\multicolumn{7}{l}{\textit{LoCoMo}} \\
32  & Deployable ($w_a{=}0$)      & \val{0.043}{0.022} & \val{0.104}{0.036} & \val{0.057}{0.026} & \val{0.965}{0.016} & \val{94.729}{61.989} \\
32  & Info-favored ($w_a{\neq}0$) & \val{0.050}{0.017} & \val{0.120}{0.031} & \val{0.066}{0.021} & \val{0.966}{0.014} & \val{99.850}{61.893} \\
64  & Deployable ($w_a{=}0$)      & \val{0.049}{0.017} & \val{0.221}{0.052} & \val{0.077}{0.025} & \val{0.982}{0.008} & \val{104.971}{68.844} \\
64  & Info-favored ($w_a{\neq}0$) & \val{0.055}{0.018} & \val{0.249}{0.055} & \val{0.087}{0.026} & \val{0.982}{0.008} & \val{115.342}{69.807} \\
128 & Deployable ($w_a{=}0$)      & \val{0.051}{0.012} & \val{0.401}{0.052} & \val{0.087}{0.020} & \val{0.988}{0.005} & \val{120.476}{68.554} \\
128 & Info-favored ($w_a{\neq}0$) & \val{0.055}{0.013} & \val{0.440}{0.052} & \val{0.095}{0.020} & \val{0.987}{0.006} & \val{137.964}{69.602} \\
\midrule
\multicolumn{7}{l}{\textit{LongMemEval}} \\
256  & Deployable ($w_a{=}0$)      & \val{0.048}{0.010} & \val{0.319}{0.138} & \val{0.073}{0.018} & \val{0.993}{0.002} & \val{817.270}{1348.513} \\
256  & Info-favored ($w_a{\neq}0$) & \val{0.049}{0.011} & \val{0.332}{0.146} & \val{0.075}{0.020} & \val{0.992}{0.002} & \val{833.750}{1352.168} \\
512  & Deployable ($w_a{=}0$)      & \val{0.042}{0.008} & \val{0.452}{0.167} & \val{0.070}{0.015} & \val{0.995}{0.002} & \val{788.716}{1355.061} \\
512  & Info-favored ($w_a{\neq}0$) & \val{0.043}{0.008} & \val{0.458}{0.158} & \val{0.071}{0.014} & \val{0.994}{0.003} & \val{801.788}{1358.805} \\
1024 & Deployable ($w_a{=}0$)      & \val{0.034}{0.005} & \val{0.542}{0.155} & \val{0.059}{0.010} & \val{0.995}{0.004} & \val{517.944}{1335.274} \\
1024 & Info-favored ($w_a{\neq}0$) & \val{0.036}{0.004} & \val{0.570}{0.131} & \val{0.061}{0.008} & \val{0.995}{0.004} & \val{570.358}{1326.036} \\
\bottomrule
\end{tabular}
\end{table}

% ---------------------------------------------------------------------
\section{Optimality Study: Full Setup and Analysis}
\label{ec:optimality}

This section provides the full setup and supporting discussion for the controlled optimality study of Section~\ref{sec:optimality}.

\noindent\textbf{Setup.}
Each instance has $n=6$ candidate memories, horizon $T=9$, and budget $B=3$; item values, sizes, demand, and reacquisition costs are drawn from the instance generator specified below, with per-step demand either \emph{stationary} or \emph{shifting} across phases. The candidate pool is sticky: a dropped memory can re-enter only through the delayed, costly reacquisition channel of Section~\ref{sec:method}. All policies observe only the current-step demand at inference and fall into two classes:
\begin{itemize}\setlength\itemsep{2pt}
\item \emph{Optimal (DP)}: the exact MDP optimum by full-horizon backward induction, enumerating budget-feasible actions and taking expectations over demand.
\item \emph{Lookahead-}$d$: a receding-horizon planner running the same DP truncated to a $d$-step window (zero terminal value), committing the first action and re-planning. Depth $d$ is the only thing that varies: $d{=}1$ is \emph{Myopic}, the exact single-step optimizer; larger $d$ ($2,4$) folds more future into the DP; at full horizon it coincides with Optimal. The ordered sequence Myopic, Lookahead-2, Lookahead-4, and Optimal is thus a single DP run at increasing foresight, isolating the value of planning depth.
\item \emph{Mixed-Score}: the deployable greedy of Section~\ref{sec:mixed-score}, ranking by online cost-penalized utility from the current step only; no planning.
\item \emph{OSL-MR (learned)}: a lightweight counterpart of our evidence learner, scoring from online features plus the Mixed-Score prior, trained offline on evidence-derived supervision, then frozen and deployed with budget-aware greedy selection (online features only, no planning).
\end{itemize}
The depth $d$ governs only how each planner \emph{decides}; the \emph{expected return} reported for every policy is computed exactly by full-horizon DP policy evaluation (no Monte-Carlo), so all numbers are directly comparable to $V^\star$.

\noindent\textbf{Rationale for the comparison set.}
The comparison set differs from the benchmark study (Section~\ref{sec:main-results}) because this study needs reference points available \emph{only} at small scale: the exact optimum $V^\star$ and the lookahead spectrum from the exact single-step optimizer up to deeper planners. The textual baselines (Recency, Generative Agents, BC-learner) have no counterpart in this synthetic numerical MDP and their roles are already settled on the real benchmarks; we keep Mixed-Score as the deployable greedy that OSL-MR generalizes.

\noindent\textbf{Instances: generator and scope.}
Concretely, each instance is one episode of horizon $T{=}9$ over a fixed set of $n{=}6$ memory items: every turn queries a single item (that turn's evidence demand), the agent retains at most $B{=}3$ units, and an item that was evicted but is queried later returns only after a $\delta$-step delay at cost $\kappa s_i$. Items have integer sizes $s_i\sim\mathrm{Unif}\{1,2,3\}$ and values $v_i\sim\mathrm{Unif}[0.5,1.5]$; the per-unit storage cost is $0.02$, the miss penalty is $1.0$, and a hit yields the item's value. Each query is drawn from a categorical demand with $P$ topic phases, each concentrating probability $0.85$ on a distinct block of $\lceil n/P\rceil$ items (each item's probability then perturbed by $\pm10\%$ per seed and renormalized) and advancing every $\lfloor T/P\rfloor$ turns; e.g., in the \emph{shifting} regime ($P{=}2$) the queried evidence comes from items $1$--$3$ early in the episode and items $4$--$6$ later, so the optimum must \emph{pre-load} the second block before the shift, whereas the \emph{stationary} regime ($P{=}1$) keeps one block throughout. The three regimes are \emph{Stationary} $(P{=}1,\delta{=}1,\kappa{=}1)$, \emph{Shifting} $(P{=}2,\delta{=}2,\kappa{=}1)$, and \emph{Shifting${}+{}$costly} $(P{=}2,\delta{=}2,\kappa{=}2.5)$; the sizes $n{=}6$, $T{=}9$ are kept small so the full-horizon optimum is exactly computable while still admitting two three-item topic blocks. These instances are not an arbitrary toy model: they instantiate the \emph{same} reward components, per-step budget, and state carry-over ($C_t=A_t$) as the full retention problem of Section~\ref{sec:method}, scaled down only enough that the optimum becomes computable. We keep the conclusions \emph{qualitative and structural}---a perfect single-step optimizer is provably suboptimal and the learned policy stays close to the true optimum---rather than transferring the specific numerical gaps to benchmark scale. Each instance is evaluated exactly by DP (no Monte-Carlo); all reported values are means over $40$ held-out test seeds (the learner trained offline on $80$ disjoint seeds), with a stationary control in which the multi-step effect is, by construction, absent.

\noindent\textbf{Learned policy.}
The learned policy is the benchmark-scale OSL-MR evidence learner: a single-hidden-layer MLP trained offline on disjoint instances and deployed frozen, with the Mixed-Score prior supplied as an input feature alongside online-observable features. The only specialization to this synthetic MDP is the form of those online features, namely the decision-time state descriptors the generator exposes: the current phase, the fractional position within it, and the item's block.

% ---------------------------------------------------------------------
\section{Computational Complexity of the Retention Problem}
\label{app:complexity}

We formally establish that the constrained memory retention problem of Section~\ref{sec:method} is NP-hard. The hardness arises already in a single decision step, from the combinatorial token-weighted hit term that rewards covering an evidence demand set through the union of retained memories' coverage sets under a storage budget.

\begin{definition}[Single-step retention, decision version]
\label{def:retention-decision}
Given a memory pool $M=\{1,\dots,n\}$ in which each memory $i$ has an integer size $s_i$ and covers an evidence subset $\mathrm{cov}(i)\subseteq U$ over a weighted evidence universe $U$ with weights $\tau(e)\ge 0$, a storage budget $B$, and a target value $W$, decide whether there exists a retained set $A\subseteq M$ with $\sum_{i\in A}s_i\le B$ such that $\sum_{e\,\in\,\bigcup_{i\in A}\mathrm{cov}(i)}\tau(e)\ge W$.
\end{definition}

\begin{proposition}
\label{prop:nphard}
The single-step retention decision problem in Definition~\ref{def:retention-decision} is NP-hard. Consequently, the multi-step constrained retention problem of Section~\ref{sec:method}, which contains the single-step problem as the special case $T=1$, is also NP-hard. The hardness also holds for the partially observable formulation because the fully observed single-step instance is a special case.
\end{proposition}

\begin{proof}{Proof.}
We reduce from the decision version of \textsc{Maximum Coverage} (NP-complete): a universe $U$, nonnegative weights $\tau(e)$, subsets $S_1,\ldots,S_n\subseteq U$, a cardinality budget $k$, and a target $W$; decide whether at most $k$ subsets have union weight at least $W$. Given such an instance, build a single-step retention instance with one memory per subset: $\mathrm{cov}(i)=S_i$, $s_i=1$, $B=k$, and single-step demand $E_1=U$ with the same weights. With $\alpha_{\mathrm{hit}}=1$ and all other coefficients zero, any budget-feasible $A$ has $R(A)=\sum_{e\in \bigcup_{i\in A}\mathrm{cov}(i)}\tau(e)$, exactly the weighted coverage of the selected subsets; hence a feasible set with reward at least $W$ exists iff the \textsc{Maximum Coverage} instance is a yes-instance. The construction is polynomial, so the problem is NP-hard.

The argument is unchanged with an active miss penalty: if every uncovered element contributes to the miss term, then $T_1^{\mathrm{miss}}=\sum_{e\in E_1}\tau(e)-T_1^{\mathrm{hit}}$, so $\alpha_{\mathrm{hit}}T_1^{\mathrm{hit}}-\alpha_{\mathrm{miss}}T_1^{\mathrm{miss}}=(\alpha_{\mathrm{hit}}+\alpha_{\mathrm{miss}})T_1^{\mathrm{hit}}-\alpha_{\mathrm{miss}}\sum_{e\in E_1}\tau(e)$, which for $\alpha_{\mathrm{hit}}+\alpha_{\mathrm{miss}}>0$ is a strictly increasing affine transform of the coverage value and preserves the maximizers and thresholds. Finally, the multi-step problem contains this as the $T=1$ case (empty initial cache, candidate pool equal to the arriving items, no future transition), so it is NP-hard; the partially observable formulation is NP-hard because the fully observed horizon-one instance is a special case.
\end{proof}

\begin{remark}[Submodularity and a greedy approximation]
\label{rem:greedy}
The per-step coverage objective $f(A)=\sum_{e\in\bigcup_{i\in A}\mathrm{cov}(i)}\tau(e)$ is monotone and submodular. Under a cardinality budget, greedy marginal-gain selection achieves the $(1-1/e)$ guarantee for monotone submodular maximization \citep{nemhauser1978analysis}; under a knapsack budget, $(1-1/e)$ is attainable by the cost-effective greedy of \citet{khuller1999budgeted} with partial enumeration of constant-size seed sets, whereas the plain utility-to-cost rule alone does not attain it. Improving the $(1-1/e)$ factor by any constant is NP-hard unless $\mathrm{P}=\mathrm{NP}$ \citep{feige1998threshold}. These guarantees presuppose known coverage sets $\mathrm{cov}(i)$, which require gold evidence unavailable at deployment; the Mixed-Score policy (Section~\ref{sec:mixed-score}) is thus a one-pass greedy ranking by the cost-penalized score $\mathrm{MS}_{t,i}$ that substitutes unobservable coverage with decision-time observable utility, forgoing the formal guarantee for deployability---the gap the evidence learner (Section~\ref{sec:method}) reduces by approximating evidence membership from logged supervision.
\end{remark}

% ---------------------------------------------------------------------
\section{Additional Robustness Analysis}
\label{ec:robustness}
The main text reports a fixed budget of 512 on LongMemEval (Figure~\ref{fig:longmemeval-sensitivity_budget512}). Figures~\ref{fig:locomo-sensitivity_budget32}--\ref{fig:longmemeval-sensitivity_budget1024} cover the remaining configurations: budgets 32/64/128 on LoCoMo and 256/1024 on LongMemEval. Each curve varies one coefficient by a multiplier $m \in \{0.25,0.5,0.75,1.0,1.5,2.0,4.0\}$ with all others at their defaults and $\alpha_{\mathrm{store}}$ fixed at $1.0$; the reward ranking and conclusions are stable across scales.

\begingroup
\setlength{\floatsep}{4pt plus 1pt minus 1pt}
\setlength{\textfloatsep}{4pt plus 1pt minus 1pt}
\setlength{\intextsep}{4pt plus 1pt minus 1pt}
\renewcommand{\textfraction}{0.02}
\renewcommand{\topfraction}{0.98}
\renewcommand{\bottomfraction}{0.98}
\renewcommand{\floatpagefraction}{0.70}

\begin{figure}[htbp]
\centering
\includegraphics[width=0.68\textwidth]{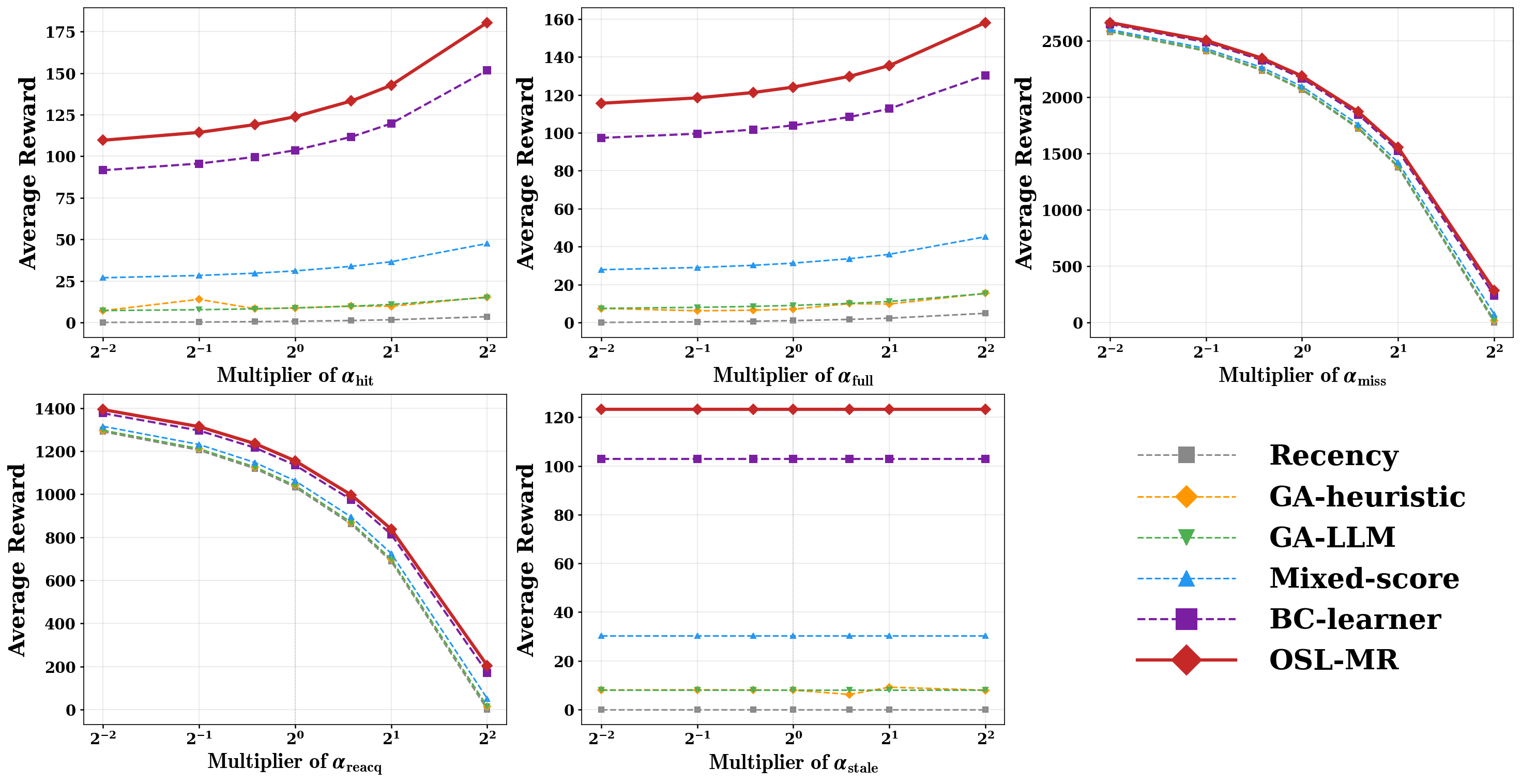}
\caption{Reward sensitivity on LoCoMo (budget 32).}
\label{fig:locomo-sensitivity_budget32}
\end{figure}

\begin{figure}[htbp]
\centering
\includegraphics[width=0.68\textwidth]{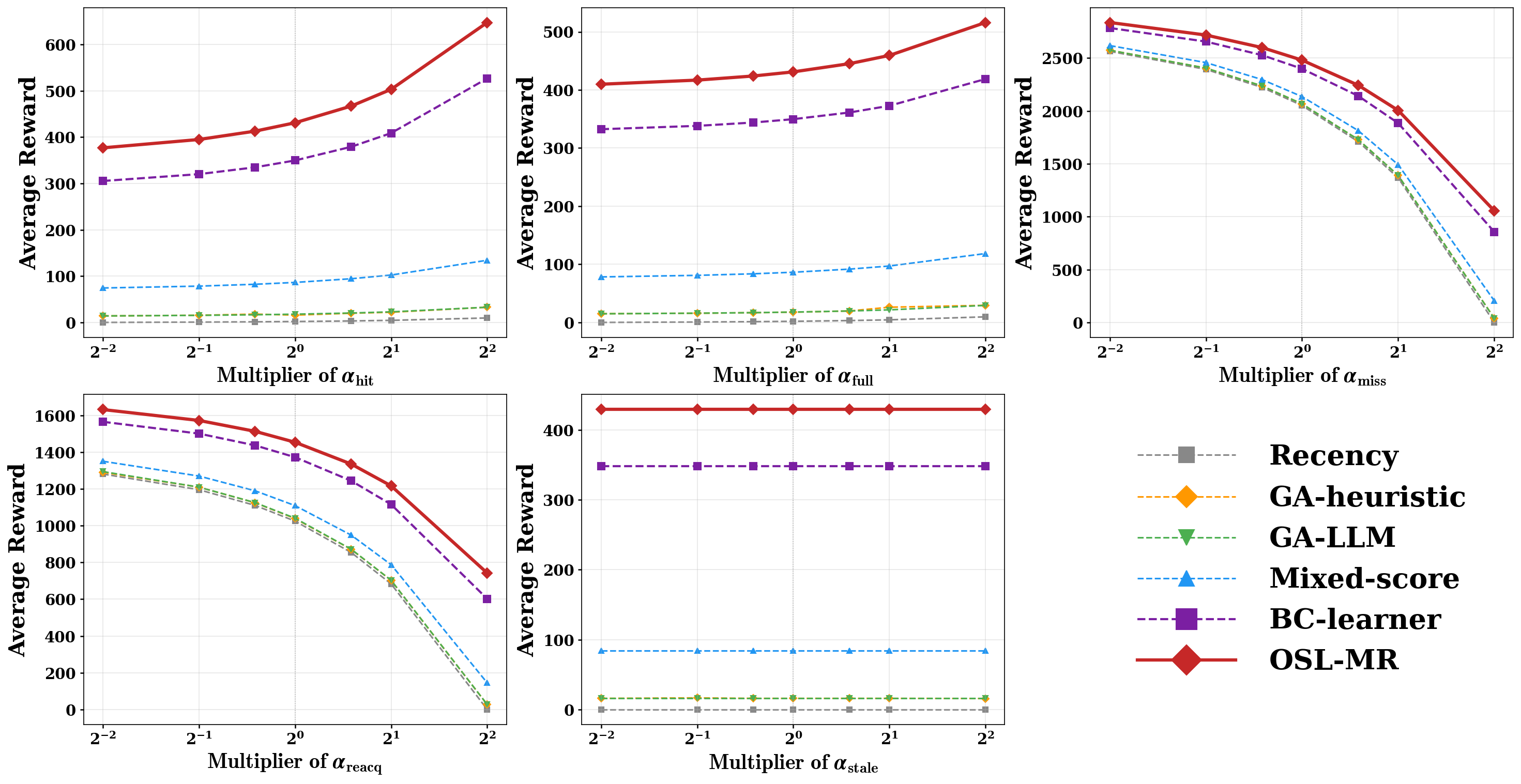}
\caption{Reward sensitivity on LoCoMo (budget 64).}
\label{fig:locomo-sensitivity_budget64}
\end{figure}

\begin{figure}[htbp]
\centering
\includegraphics[width=0.68\textwidth]{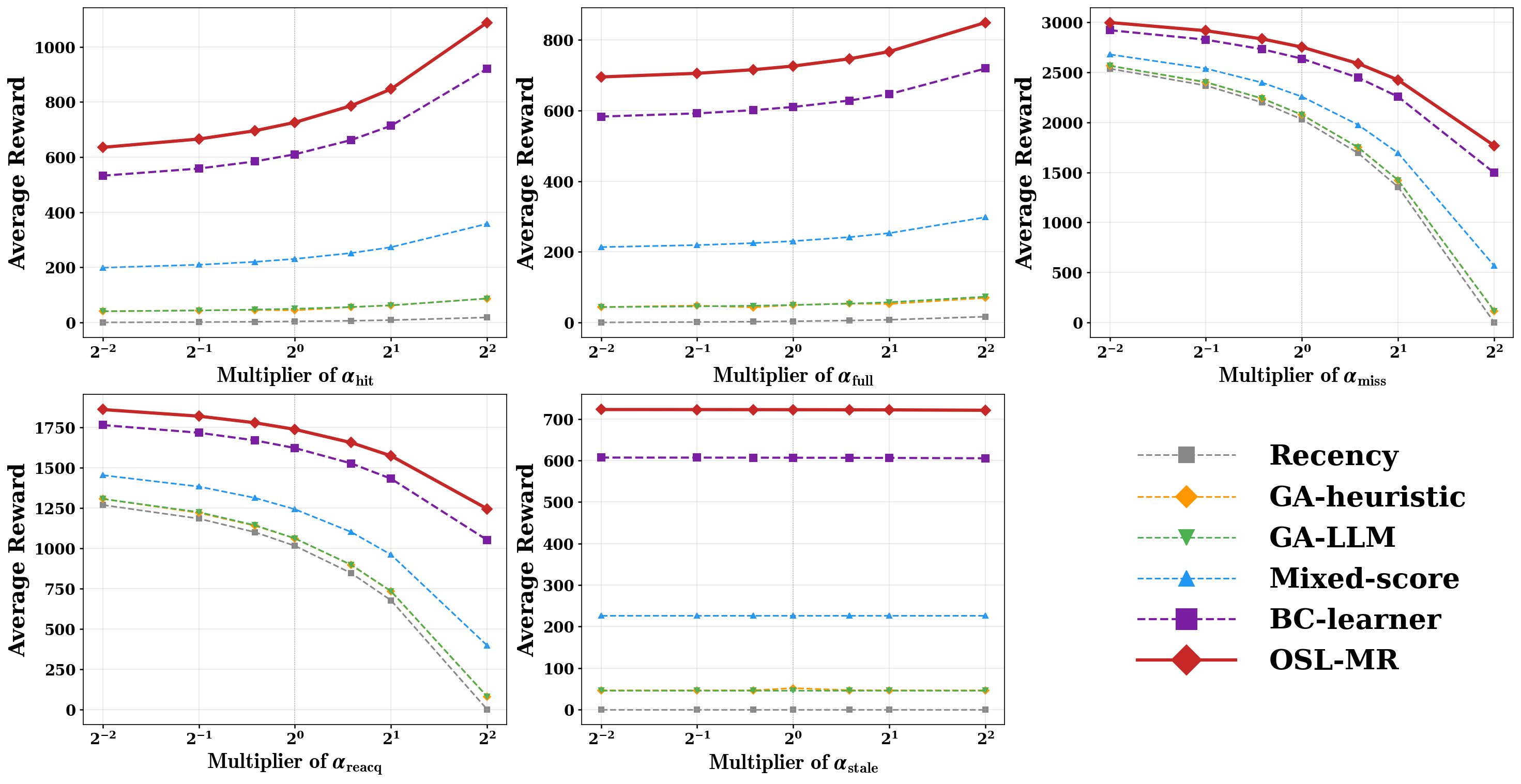}
\caption{Reward sensitivity on LoCoMo (budget 128).}
\label{fig:locomo-sensitivity_budget128}
\end{figure}

\begin{figure}[htbp]
\centering
\includegraphics[width=0.68\textwidth]{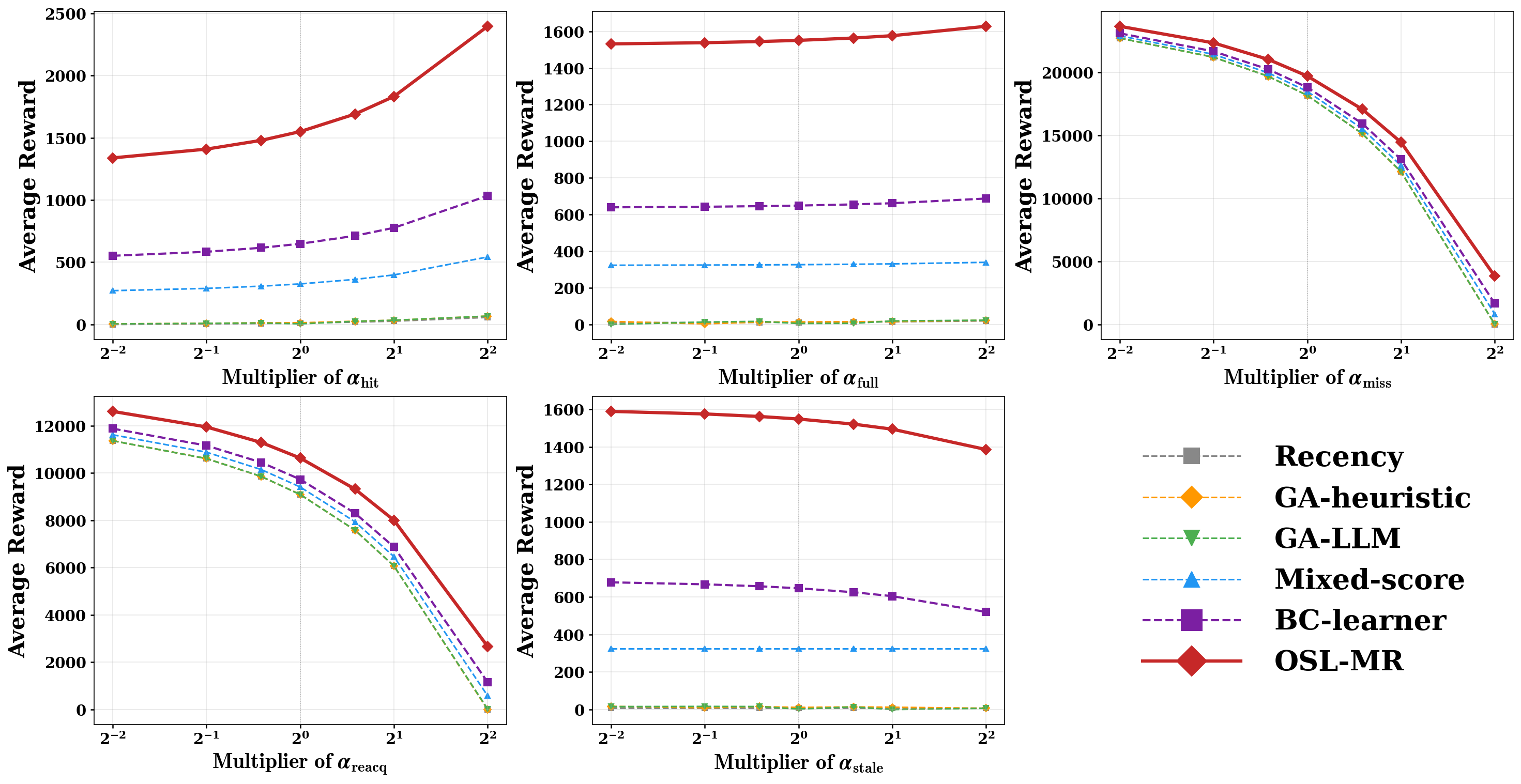}
\caption{Reward sensitivity on LongMemEval (budget 256).}
\label{fig:longmemeval-sensitivity_budget256}
\end{figure}

\begin{figure}[htbp]
\centering
\includegraphics[width=0.68\textwidth]{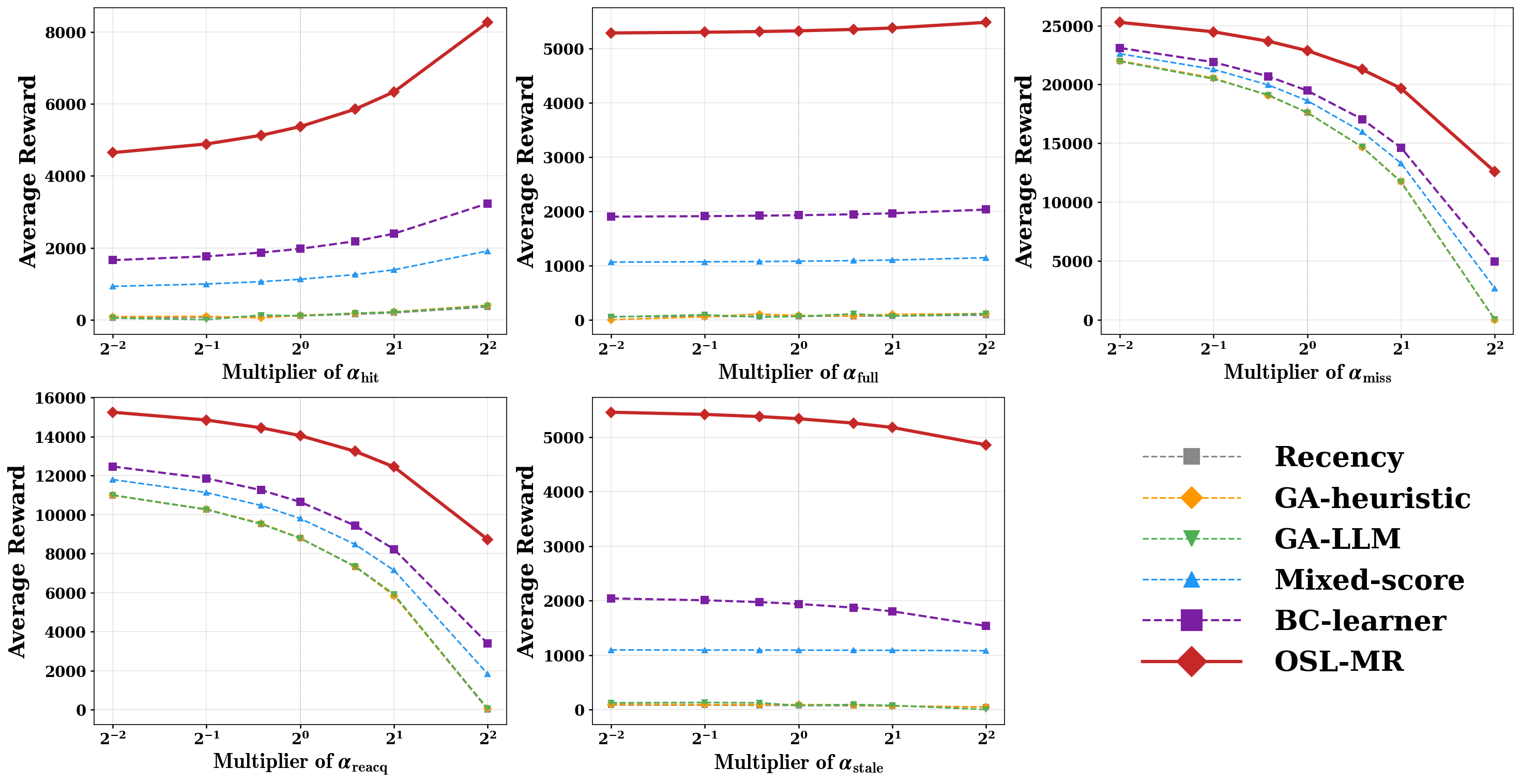}
\caption{Reward sensitivity on LongMemEval (budget 1024).}
\label{fig:longmemeval-sensitivity_budget1024}
\end{figure}
\endgroup

% ---------------------------------------------------------------------
\section{Freshness Proxy and Related Hyperparameters}
\label{app:freshness-params}

\paragraph{Memory-type taxonomy and labeling.}
Each memory $i$ has a fixed binary type $\mathrm{type}_i \in \{\text{`stable'}, \text{`temporary'}\}$ at creation. Offline, types come from an LLM classifier (DeepSeek-V4-Flash) prompted with the target utterance and a $\pm 3$-turn window, emitting a single JSON field \emph{memory\_type} (temperature $0$, at most $20$ tokens). `stable' covers long-valid content (identity, enduring preferences, settled relationships, lasting facts); `temporary' covers short-horizon content (plans, greetings, transient moods, ongoing intents). When LLM labels are unavailable (online-safe evaluation), a content-only heuristic labels as `temporary' utterances with planning cues (plan, tomorrow, will, tonight), short greetings (fewer than 100 characters), or present-tense cues (just, currently), and `stable' otherwise.

\paragraph{Decay dynamics.}
Let $\Delta t_i = t - t_i^{\mathrm{store}}$ be the steps since item $i$ was written (computed as $\text{anchor\_step}-\text{item\_id}$). The time-scaling factor is $\kappa(\Delta t) = 1 + \log(1 + \Delta t/\tau)$ ($\tau>0$), and the effective decay rate and freshness proxy are
\[
\gamma_i = \gamma^{\mathrm{base}}(\mathrm{type}_i)\cdot \kappa(\Delta t_i),
\qquad
\phi_{t,i} = \exp(-\gamma_i \Delta t_i),
\]
with type-specific base rates $\gamma^{\mathrm{base}}(\mathrm{type})$ below; stable memories decay slowly, temporary ones faster. We set $\mathrm{stale\_risk}_{t,i} = 1 - \phi_{t,i}$, used as an input feature for the learned policies (Section~\ref{sec:evidence-learning}); a memory is treated as \emph{stale} once $\phi_{t,i}$ falls below $0.5$, the threshold used by the stale term $T_t^{\mathrm{stale}}$ of Section~\ref{sec:method}. At creation $\phi_{t,i}=1$; as $\Delta t_i$ grows, temporary items decay toward $0$ while stable ones stay near-fresh (e.g., at $\Delta t_i=50$, $\approx 0.76$ for `stable' vs.\ $\approx 0$ for `temporary').

\noindent\textbf{Parameter calibration.}
All values are calibrated separately on the development splits of LoCoMo and LongMemEval:
\begin{itemize}\setlength\itemsep{2pt}
    \item $\tau = 10$, $\gamma^{\mathrm{base}}(\text{`stable'}) = 0.002$, $\gamma^{\mathrm{base}}(\text{`temporary'}) = 0.15$.
    \item Deployable Mixed-Score weights ($w_a=0$), LoCoMo: $w_r=0.22$, $w_q=0.56$, $w_a=0$, $w_e=0.12$, $w_s=0.06$, $w_u=0.02$, $w_c=0.06$ (the answer-overlap mass is reassigned to question overlap during recalibration).
    \item Deployable Mixed-Score weights ($w_a=0$), LongMemEval: $w_r=0.25$, $w_q=0.55$, $w_a=0$, $w_e=0.12$, $w_s=0.05$, $w_u=0.03$, $w_c=0.07$.
    \item For reference, the information-favored variant ($w_a\neq0$; not deployable, Section~\ref{ec:prior-ablation}) calibrates to $w_a=0.15$ on LoCoMo ($w_r=0.35$, $w_q=0.40$, $w_e=0.05$, $w_s=0.03$, $w_u=0.02$, $w_c=0.05$) and $w_a=0.08$ on LongMemEval ($w_r=0.22$, $w_q=0.48$, $w_e=0.12$, $w_s=0.06$, $w_u=0.02$, $w_c=0.06$).
\end{itemize}

% ---------------------------------------------------------------------
\section{LLM Prompts for Memory Annotation}
\label{app:llm-prompts}
We annotate each utterance with a memory type and an importance score using a chat model at temperature~$0$. Each prompt receives a \emph{context block} of up to three turns before and after the target utterance, formatted as [dia\_id] speaker: text with the target line prefixed by \textgreater{}\textgreater{}\textgreater; runtime placeholders \{context\_block\}, \{dia\_id\}, \{speaker\}, \{text\} are filled at call time.

\vspace{12pt}
\begin{promptbox}{Memory Annotation Prompt Templates}
\footnotesize
\label{app:prompt-memory-type}\label{app:prompt-importance}
\textit{\textbf{(a) Memory type (stable versus temporary).}}
\textbf{System.} Memory-type annotator; output only \{\textquotedbl memory\_type\textquotedbl:\textquotedbl stable\textquotedbl\} or \{\textquotedbl memory\_type\textquotedbl:\textquotedbl temporary\textquotedbl\}.
\textbf{User.} Label the target utterance in context: stable = long-lasting (identity, long-term preferences, stable relationships, facts that still hold); temporary = short-lived (plans, current emotions, small talk, ongoing events, unfulfilled intentions). Inputs: \{context\_block\}, \{dia\_id\}, \{speaker\}, \{text\} (max\_tokens $=20$).

\vspace{3pt}\hrule\vspace{3pt}
\textit{\textbf{(b) Memory importance (1--10).}}
\textbf{System.} Rate importance for long-term recall; return only \{\textquotedbl importance\textquotedbl: \textless integer 1--10\textgreater\} (1 = mundane, 10 = life-changing), clipped to $[1,10]$.
\textbf{User.} Rate the target on 1--10 (e.g.\ ``cleaning up the room'' $\rightarrow$ 2; ``asking your crush out on a date'' $\rightarrow$ 8). Inputs: \{context\_block\}, \{dia\_id\}, \{speaker\}, \{text\} (max\_tokens $=24$).
\end{promptbox}

%%%%%%%%%%%%%%%%%
\end{document}